\documentclass{article}

\usepackage{arxiv}

\usepackage[utf8]{inputenc} 
\usepackage[T1]{fontenc}    
\usepackage{hyperref}       
\usepackage{url}            
\usepackage{booktabs}       
\usepackage{amsfonts}       
\usepackage{nicefrac}       
\usepackage{microtype}      
\usepackage{lipsum}
\usepackage{graphicx}
\graphicspath{ {./images/} }

\usepackage{bm}

\newcommand{\bmt}[1]{\tilde{\bm{#1}}}

\usepackage{algorithm}
\usepackage{algorithmic}
\usepackage{subcaption}
\usepackage{amsmath,amsthm,amssymb}
\DeclareUnicodeCharacter{211D}{\mathbf{R}}

\newtheorem{proposition}{Proposition}
\newtheorem{lemma}{Lemma}
\newtheorem{theorem}{Theorem}

\title{Prescriptive PCA: Dimensionality Reduction for Two-stage Stochastic Optimization}

\author{
 Long He \\
  School of Business\\
  George Washington University\\
  Washington, DC 20052 \\
  \texttt{longhe@gwu.edu} \\
   \And
 Ho-Yin Mak \\
  McDonough School of Business\\
  Georgetown University\\
  Washington, DC 20057 \\
  \texttt{hoyin.mak@georgetown.edu} \\
}

\begin{document}
\maketitle
\begin{abstract}
In this paper, we consider the alignment between an upstream dimensionality reduction task of learning a low-dimensional representation of a set of high-dimensional data and a downstream optimization task of solving a stochastic program parameterized by said representation. In this case, standard dimensionality reduction methods (e.g., principal component analysis) may not perform well, as they aim to maximize the \emph{amount} of information retained in the representation and do not generally reflect the \emph{importance} of such information in the downstream optimization problem. To address this problem, we develop a \emph{prescriptive} dimensionality reduction framework that aims to minimize the degree of suboptimality in the optimization phase. For the case where the downstream stochastic optimization problem has an expected value objective, we show that prescriptive dimensionality reduction can be performed via solving a distributionally-robust optimization problem, which admits a semidefinite programming relaxation. Computational experiments based on a warehouse transshipment problem and a vehicle repositioning problem show that our approach significantly outperforms principal component analysis with real and synthetic data sets. 
\end{abstract}


\section{Introduction}

Common practice of data analytics to business planning and operations usually involves a \emph{pipeline} structure, consisting of a sequential set of processes converting collected raw data to business decisions to be implemented. For prescribing operational decisions, mathematical programming models (e.g., for production planning) often fit in the downstream stages of the pipeline, with inputs fed from upstream processes, e.g., learning statistical models for product demand with machine learning methods. This approach where learning and optimization are performed separately and sequentially, while intuitive to implement, can be suboptimal as the learning phase often does not account for how its outputs are used as inputs in the downstream optimization phase \cite{elmachtoub2022smart}.

In this paper, we consider a prescriptive analytics pipeline that aims to prescribe a set of optimal decisions in a stochastic optimization problem given high-dimensional input data. Specifically, we consider two sequential phases: (1) a dimensionality reduction phase that learns a low-dimensional representation of the probability distribution that generates the high-dimensional data; and (2) a stochastic optimization phase that prescribes an optimal solution to a stochastic program whose parameter uncertainty is governed by said distribution. Such a scenario has a variety of practical applications. For example, an urban mobility platform may optimize repositioning of its fleet based on a probability distribution of travel demand learned from data. In such applications, while the raw origin-destination travel data could be high dimensional, the variations in the underlying demand distribution are often governed by a low-dimensional structure involving a smaller number of factors.  

Our research objectives are as follows. First, we investigate the limitation of using standard dimensionality reduction methods (such as principal component analysis, PCA) in such a two-phase scenario. Particularly, we demonstrate that standard PCA can fail to identify factors that are relevant to the downstream stochastic program. 
Second, observing the shortcoming of PCA in aligning the two phases of the problem, we propose a prescriptive alternative that learns a low-dimensional representation of data that minimizes downstream suboptimality. Using distributionally-robust optimization techniques, we show that the problem of minimizing a proxy (upper bound) for the degree of suboptimality in the downstream stochastic program can be formulated as a bi-convex problem, to which a local optimal solution can be found via an alternating algorithm that involves solving a semidefinite program in each iteration. Third, using synthetic and real data, we investigate the effectiveness of our approach based on a joint production and inventory allocation problem for a supply chain network, and a vehicle repositioning problem for an urban mobility platform. 

\subsection{Related Literature}
Our work is related to the literature on mathematical programming and machine learning.  

\emph{Interface of Machine Learning and Optimization.}
Aligned with the prevalent application of machine learning, numerous researchers have studied end-to-end techniques that integrate machine learning and mathematical programming in prescriptive settings. A prevalent approach that integrates machine learning and mathematical programming in prescriptive settings is ``predict-then-optimize'', which involves first making accurate predictions from data, using machine learning tools, and then solving optimization problems taking such predictions as input \cite{ferreira2016analytics,glaeser2019optimal}. 
Noting that the criteria for improving predictions and improving decisions are often not aligned, a growing stream of recent work looks into more integrated, end-to-end approaches to guide decisions directly from historical data that leverage contextual information \cite{ban2019big,ban2019dynamic,bertsimas2020predictive}. 
To address the potential misalignment between the loss function for a predictive model and the objective function in the downstream optimization model, \cite{elmachtoub2022smart} define suitable loss functions that take the downstream optimization problem into account when measuring errors in predictions. They further elaborate on how to train predictive models using these loss functions under this smart predict-then-optimize framework. \cite{el2019generalization} further provide generalization bounds for the framework. These mentioned papers focus on integrating predictive models, trained through supervised learning, with a downstream optimization phase. Our paper, on the other hand, considers the upstream phase of dimensionality reduction, a class of unsupervised learning. As opposed to making accurate predictions to inform the downstream optimization task, we aim to identify a low-dimensional space of features that is informative for the downstream optimization problem.

\emph{Dimensionality Reduction and Optimization.}
PCA has been adopted as the standard approach for linear dimensionality reduction given a sample covariance matrix \cite{eckart1936approximation}. 
Casting the dimensionality reduction problem as one of minimizing the error in approximating a matrix subject to structural (including rank) constraints, researchers have proposed mathematical-programming-based approaches to different variants of the problem \cite{d2004direct,d2008optimal,bertsimas2022solving,bertsimas2017certifiably}. 
Although the studies mentioned above take a mathematical programming approach as in our paper, they consider the classic objective of minimizing reconstruction error. As we argue later, when the low-dimensional model is fed into a subsequent stochastic programming problem, this objective of error minimization is not necessarily aligned with one of identifying good solutions in the sense of the downstream objective function. To address this issue, \cite{kao2013learning} propose the directed PCA approach that estimates the covariance matrix by balancing between the PCA objective and empirical optimization. While its aim is similar to ours, both our setting and methodological approach are significantly different. In particular, \cite{kao2013learning} consider a downstream single-stage stochastic convex optimization problem with uncertain objective coefficients where the decision variables are unconstrained; in contrast, we consider a downstream two-stage stochastic program with recourse decisions that depend on the realization of uncertainty. Therefore, the low-dimensional representation governs the space of recourse decisions in the downstream problem. Further, the presence of first-stage decisions in our setting requires an approach that yields a covariance matrix that is without using the mean (i.e., invariant to the first-stage decisions) of the data. This is one of the motivations for us to use distributionally-robust optimization rather than Bayesian optimization as in \cite{kao2013learning}.

\emph{Decision Rules in Two-Stage Optimization under Uncertainty.}
Our setting concerns learning a low-dimensional representation to characterize the uncertainty associated with a two-stage stochastic program with recourse. Conceptually, this objective closely links with the literature on decision rules in multi-stage optimization with recourse. In this literature, instead of allowing recourse decisions to be optimized to the specific realizations of uncertainty, they are confined to the space of parametric functions, known as decision rules, of the realized uncertain parameters. This enables the problem to be (heuristically) solved by optimizing over the space of parameters of these decision rule functions. The decision rule approach has been adopted in both robust optimization and stochastic programming settings \cite{ben2004adjustable,bertsimas2010optimality,chen2008linear,see2010robust,goh2010distributionally,kuhn2011primal}. 
Various studies in the literature have pointed out the importance of careful parameterization of the primitive uncertainties of the problem in enhancing the performances of LDRs and their generalizations \cite{chen2009uncertain,zhen2018adjustable,bertsimas2019adaptive}. 
 Similarly to these works, our analysis suggests that learning an appropriate initial representation from data can help significantly improve the performance of linear decision rules in a data-driven stochastic programming setting. 

\section{Dimensionality Reduction for Stochastic Optimization}\label{sec:dim_red}
We consider a two-stage stochastic program with recourse:
\begin{eqnarray}
\min_{ \mathbf{x} \in \mathbf{X}}&& \mathbf{c}^T \mathbf{x} + E[h(\bmt{z} - \mathbf{D}^T \mathbf{x} )] \label{eq:1st_stage} \\
\mbox{ where} && h(\bm{z}) = \min_{\mathbf{y}}  \mathbf{b}^T \mathbf{y}, \mbox{ s.t. } \mathbf{A} \mathbf{y} \ge   \bm{z}. \label{eq:2nd_stage}
\end{eqnarray}
In \eqref{eq:1st_stage}, the first-stage decision variables $\mathbf{x}$ are chosen under uncertainty, as characterized by the random variable $\bmt{z}$ (in $\mathbb{R}^n$); then, once the values for $\bmt{z}$ are realized, the second-stage decisions $\mathbf{y}$ are chosen to optimize the recourse problem \eqref{eq:2nd_stage}. We assume that the problem has complete recourse, i.e., \eqref{eq:2nd_stage} is feasible for any value of $\bm{z}$. This implies strong duality:
\begin{equation}
h(\bm{z}) = \max_{\mathbf{w} \ge 0} \mathbf{w}^T \bm{z}, \mbox{ s.t. } \mathbf{A}^T \mathbf{w} = \mathbf{b}.\label{eq:2nd_stage_dual}
\end{equation}
We consider the case where $n$ is large, i.e., $\bmt{z}$ resides in a high-dimensional space. 
In a data-driven setting, the distribution of $\bmt{z}$ is unknown; Instead, a set of training data is available. Let $\boldsymbol{\mu}$ and $\boldsymbol{\Sigma}$ denote estimates of the mean and covariance matrix from data (e.g., the sample mean and covariance). 
With high dimensionality of $\bmt{z}$, it is common practice to model it with a low-dimensional factor-based representation. It is known that evaluating the expectation of the random objective function for a stochastic program is $\# P$-hard \cite{hanasusanto2016comment}. As a computationally-efficient approach, decision rule models consider the $n$-dimensional uncertain problem parameters to be linearly dependent on a set of $k$  \emph{primitive uncertainties} or \emph{factors} \cite{chen2008linear, goh2010distributionally}, where $k << n$, and aim to optimize decision rules defined on said factors. Thus, the effectiveness of decision rule approaches critically depends on identifying such a low-dimensional factor model that closely represents the uncertainties pertaining to the original problem.


An intuitive approach would be to apply a standard dimensionality reduction algorithm on $\bmt{z}$  (such as PCA) and then feeding the resulting model to the stochastic program \eqref{eq:1st_stage}. However, this na{\"i}ve sequential approach would not perform well generally, because the dimensionality reduction algorithm does not take into account the downstream optimization task. For example, one may apply PCA to identify the rank-$k$ projection that captures the maximal amount of variance in the data. Intuitively, this corresponds to finding the $k$ basis directions along which the data exhibits the largest variation; However, these are not generally the most \emph{relevant} directions of variation for the downstream stochastic program (e.g., for defining effective decision rules). To address this limitation, we propose a \emph{prescriptive} dimensionality reduction framework that identifies a low-dimensional projection of the data that minimizes a measure of suboptimality in the downstream stochastic program. 

\subsection{The Limitation of PCA}

To illustrate, we consider the evaluation of the downstream stochastic program's objective (or more specifically, the component that depends on the model of uncertainty, i.e., the recourse objective $h(\cdot)$) based on a projection of the data onto some lower-dimensional subspace. In particular, suppose each data point $\bm{z}$ is projected onto a $k$-dimensional subspace as $\hat{\bm{z}} = \mathbf{V} \mathbf{V}^T \bm{z}$ where $\mathbf{V} \in \mathbb{R}^{n \times k}$ and $rank(\mathbf{V}) = k$. Note that $ \mathbf{V} \mathbf{V}^T$ is a symmetric $n \times n$ matrix with rank $k$. For example, in the case of PCA, we have $\mathbf{V} = \mathbf{V}_{[k]}$, the $n \times k$ matrix whose columns correspond to the eigenvectors associated with the $k$ largest eigenvectors. The second-stage objective value under the projected data is 
\begin{eqnarray}
h(\hat{\bm{z}}) = h(\mathbf{V} \mathbf{V}^T \bm{z}) &=& \max_{\mathbf{w} \ge 0} \mathbf{w}^T \mathbf{V} \mathbf{V}^T \bm{z}, \mbox{ s.t. } \mathbf{w} \in \boldsymbol{P},  \label{eq:2nd_stage_proj}
\end{eqnarray}
\noindent where the polyhedron $\boldsymbol{P} = \{\mathbf{w} \ge 0 | \mathbf{A}^T \mathbf{w} = b \}$.
Then, the following suggests that the second-stage objective value evaluated under the projected data is equivalent to the optimal objective value of a counterpart problem defined over the projected feasible region: $h(\hat{\bm{z}}) = \{\max  \hat{\mathbf{w}}^T \bm{z}, \mbox{ s.t. } \hat{\mathbf{w}} \in \hat{\mathbf{P}} \}$ where $\hat{\boldsymbol{P}} = \{ \hat{\mathbf{w}} | \hat{\mathbf{w}}=  \mathbf{V} \mathbf{V}^T \mathbf{w},  \mathbf{w} \in \boldsymbol{P} \}$.

Under PCA, the data is projected onto the $k$-eigenspace of the covariance matrix. Thus, when the (dual) problem is not \emph{aligned} with said eigenspace, the PCA solution could perform badly. In particular, if the projected polyhedron $\hat{\mathbf{P}}$ is orthogonal to the first $k$ eigenvectors of $\boldsymbol{\Sigma}$ (i.e., the columns of $\mathbf{V}_{[k]}$), the recourse objective under the PCA projection will have $h(\hat{\bm{z}}) \equiv 0$ for all $\bm{z}$. That is, the PCA solution may yield a projection that, while capturing the maximum \emph{amount} of variation in the data, fails to capture any  \emph{relevant} variation in terms of optimizing the second-stage problem. This occurs if the data is projected onto a subspace (the $k$-eigenspace) that is orthogonal to the \emph{dual} feasible region of the recourse problem.

\subsection{Prescriptive PCA} \label{sec:PPCA}
To address the above limitation of PCA, we propose an alternative to PCA, which we refer to as prescriptive PCA (PPCA) that aligns with the downstream stochastic program. Formulating the prescriptive PCA problem as a mathematical program, we will show that a distributionally-robust bound on the expected reconstruction error can be computed by solving semidefinite programs. 

Following the previous discussion, we seek a projection $\mathbf{V} \mathbf{V}^T \bmt{z}$ that yields a small expected reconstruction error (or loss) in terms of the second-stage objective value, i.e.,
\begin{equation*}
L( \mathbf{V}) = \left|E[h(\bmt{z})] - E[h( \mathbf{V} \mathbf{V}^T \bmt{z})] \right|.
\end{equation*}
To this end, we derive an upper bound on $L( \mathbf{V})$ that can be computed efficiently.

Recall that we seek an approximation independent of the mean of $\bmt{z}$, denoted $\boldsymbol{\mu}$. Let $\bmt{z}_0 = \bmt{z} - \boldsymbol{\mu}$ be the centered random variable. We seek an upper bound on $L( \mathbf{V})$ that only depends on $\bmt{z}_0$, but not $\boldsymbol{\mu}$. Following \eqref{eq:2nd_stage_proj}, we have:
\begin{proposition} \label{prop:subadditive}
	Suppose the linear program \eqref{eq:2nd_stage} has complete recourse. Let $\bm{z}_0$ be a realization of  $\bmt{z}_0$. 	
	Then, for any $\bm{z}_1, \bm{z}_e$ such that $\bm{z}_1 + \bm{z}_e = \bm{z}_0$, it holds that:
	$$h(\boldsymbol{\mu} + \bm{z}_0) \le h(\boldsymbol{\mu} +  \bm{z}_1) + h(\bm{z}_e) \mbox{, and } h(\boldsymbol{\mu} + \bm{z}_1) \le h(\boldsymbol{\mu} + \bm{z}_0) + h(-\bm{z}_e).$$
\end{proposition}
\begin{proof} We note that
    \begin{align*}
h(\boldsymbol{\mu} + \bm{z}_0) 
&= \max_{j \in \{1,\cdots, J\}} \left[\mathbf{w}_j^T (\mathbf{\mu} + \bm{z}_1) + \mathbf{w}_j^T \bm{z}_e \right]\\
&\le  \max_{j \in \{1,\cdots, J\}} \mathbf{w}_j^T (\mathbf{\mu} + \bm{z}_1) + \max_{j \in \{1,\cdots, J\}} \mathbf{w}_j^T  \bm{z}_e \\
&=  h(\boldsymbol{\mu} +  \bm{z}_1) + h(\bm{z}_e) .
\end{align*}
Because $\bm{z}_1   = \bm{z}_0- \bm{z}_e$, it follows similarly that $h(\boldsymbol{\mu} + \bm{z}_1) \le h(\boldsymbol{\mu} +  \bm{z}_0) + h(-\bm{z}_e)$ \qedhere
\end{proof}

Proposition \ref{prop:subadditive} implies that the error in evaluating the objective value under the approximation is bounded by the objective value evaluated under the error term. More specifically, the expected approximation error is bounded as follows.
\begin{proposition} \label{prop:bound_error}
Consider an approximation $\bmt{z}_0 \approx \bmt{z}_1$, with error $\bmt{z}_e = \bmt{z}_0 - \bmt{z}_1$ (with probability one). The absolute error on the evaluated expectation of the recourse problem is bounded above by:
\begin{eqnarray}
\left| E[h(\boldsymbol{\mu} + \bmt{z}_1)]-  E[h(\boldsymbol{\mu} + \bmt{z}_0)] \right| \le \max\left\{  
	E[h(\bmt{z}_e)],  E[h(-\bmt{z}_e)]\right\}. \label{eq:bound_error}
\end{eqnarray} 
\end{proposition}
\begin{proof}
    Following Proposition \ref{prop:subadditive}, it holds that:
$E[h(\boldsymbol{\mu} + \bmt{z}_0) -  h(\boldsymbol{\mu} + \bmt{z}_1)] \le E[h(\bmt{z}_2)]$ and $E[h(\boldsymbol{\mu} + \bmt{z}_1) -  h(\boldsymbol{\mu} + \bmt{z}_0)] \le E[h(-\bmt{z}_2)]$, \mbox{ which imply \eqref{eq:bound_error}.} 
\qedhere
\end{proof}

The inequality \eqref{eq:bound_error} bounds the error of the approximating $E[h(\boldsymbol{\mu} + \bmt{z}_0)]$ by $E[h(\boldsymbol{\mu} + \bmt{z}_1)]$. This bound can serve as a surrogate of the degree of suboptimality of a chosen approximation $\bmt{z}_1$ to $\bmt{z}_0$ that satisfies desired structural properties (e.g., has a low-rank covariance matrix). A key feature of this bound is that it is independent of the mean $\boldsymbol{\mu}$, and thus is also invariant up to the addition of any constant to $\bmt{z}$. This is a critical property, as approximating the stochastic program \eqref{eq:1st_stage} requires an approximation for $E[h(\bmt{z} - \mathbf{D}^T \mathbf{x})]$ that holds for any first-stage decision $\mathbf{x}$ (where the deterministic term $\mathbf{D}^T \mathbf{x}$ can be absorbed in the mean).

The bound \eqref{eq:bound_error} is not necessarily tight in general. In particular, if the random variable $\boldsymbol{\mu} + \bmt{z}_0$ is split into two components, $\boldsymbol{\mu} + \bmt{z}_1$ and $\bmt{z}_e$ with comparable weights, it is conceivable that the bound is loose. On the other hand, the bound becomes tighter if the approximation $\boldsymbol{\mu} + \bmt{z}_1$ carries dominant weight compared with the residual term $\bmt{z}_e$, which tends to be the case as the right-hand side of \eqref{eq:bound_error} are to be minimized. Next, we show that a tight bound on this surrogate can be computed efficiently, and propose an algorithm to minimize it.

\subsection{Distributionally-Robust Bound}
Proposition \ref{prop:bound_error} suggests a surrogate for the reconstruction error in approximating the expected second-stage objective value under a given projection $\bmt{z}_1 = \mathbf{V} \mathbf{V}^T \bmt{z}_0$, by evaluating the expected value of the second-stage objective under the residual $\bmt{z}_0$.
Ideally, evaluating the bound in Proposition \ref{prop:bound_error} requires knowledge of the distribution of $\bmt{z}_0$. 
However, in practice, it is desirable to evaluate this bound without making distributional assumptions. Invoking results from the distributionally-robust optimization literature, we show that a bound on this expected value can be computed by solving a semidefinite program. This bound enables us to formulate a parsimonious PPCA procedure that uses only the covariance matrix as a sufficient statistic. We use the following as a direct result of Theorem 2 in \cite{natarajan2017reduced}.
\begin{proposition}\label{prop:persistency}
Suppose the mean and covariance matrix of $\bmt{z}$ are given by $\boldsymbol{\mu}$ and $\boldsymbol{\Sigma}$, respectively. Then, under any distribution of $\bmt{z}$, $E[h(\bmt{z})] \le \bar{h}(\boldsymbol{\mu},\boldsymbol{\Sigma})$, where:
\begin{equation}
	\begin{array}{r@{}l}
		\bar{h}(\boldsymbol{\mu}, \boldsymbol{\Sigma}) \equiv  
		\max_{\mathbf{p,Y,X}} \quad &tr(\mathbf Y)\\
		\mbox{s.t.} \quad &\mathbf{A}^T \mathbf{p}=\mathbf{b}\\
		&diag(\mathbf{A}^T \mathbf{X} \mathbf{A})=\mathbf{b}^2\\
		&\begin{pmatrix} 
			1 &   \boldsymbol{\mu}^T & \mathbf p^T\\
			\boldsymbol{\mu} & \boldsymbol{\Sigma}  & \mathbf{Y}^T \\
			\mathbf p &\mathbf{ Y} & \mathbf{X}
		\end{pmatrix} \succeq 0\\
		& \mathbf{p,X} \ge 0
	\end{array} \label{eq:persistency}
\end{equation}
\end{proposition}

In Proposition \ref{prop:persistency}, problem  \eqref{eq:persistency} is a relaxation of the tight upper bound proved in Theorem 2 of \cite{natarajan2017reduced}. Using this result, we can bound the expected value on the right hand side of \eqref{eq:bound_error} by $\bar{h}(\mathbf{0}, \boldsymbol{\Sigma}_e)$, where $\boldsymbol{\Sigma}_e$ is the covariance matrix of the residuals $\bmt{z}_e$ (which has zero mean by construction). Note that both $\bmt{z}_e$ and $-\bmt{z}_e$ have the same mean (zero) and covariance matrix $\boldsymbol{\Sigma}_e$, thus applying the bound gets rid of the max operator in \eqref{eq:bound_error}. 
This yields a \emph{distribution-free} performance bound on any approximation $\bmt{z}_0 \approx \bmt{z}_1$.

Then, the best (e.g., low-rank) approximation can be obtained by minimizing the performance bound. Specifically, we look for a rank-$k$ projection $\bmt{z}_1 = \mathbf{V} \mathbf{V}^T \bmt{z}_0$ that minimizes the distribution-free performance bound. In the spirit of parsimony, we operate over the space of covariance matrices. Thus we focus on the covariance matrix of the projected data, $\boldsymbol{\Sigma}_1$. 
Note that its eigenvalue decomposition is given by  $\boldsymbol{\Sigma}_1 = \mathbf{V} \mathbf{E}_k \mathbf{V}^T $, where $\mathbf{E}_k$ is a $k \times k$ matrix with positive values in the diagonal entries and zero elsewhere. Thus, the problem of optimizing over the $n \times k$ matrix $\mathbf{V}$ is equivalent to optimizing over  $\boldsymbol{\Sigma}_1$ subject to $rank( \boldsymbol{\Sigma}_1 ) \le k$. Then, we can formulate the PPCA problem for dimensionality reduction as:
\begin{eqnarray}
		\min_{\mathbf{\Sigma}_1 ,  \mathbf{\Sigma}_e} \quad &&  \bar{h}(\mathbf{0}, \mathbf{\Sigma}_e) +  \theta \langle \mathbf{\Sigma}_1, \mathbf{\Sigma}_e \rangle \label{eq:min-max} \\
		\mbox{s.t.} \quad && \mathbf{\Sigma}_1 + \mathbf{\Sigma}_e = \mathbf{\Sigma}_0 \label{eq:Sigma}\\
		&&  \mathbf{\Sigma}_1 ,  \mathbf{\Sigma}_e \succeq 0 \label{eq:PSD} \\
		&& \mathbf{\Sigma}_1 \in \mathbb{W}, \label{eq:rank}
\end{eqnarray}
where $\mathbb{W}=\left\{\mathbf S\in \mathbb R^{n\times n}: \mbox{rank}(\mathbf S)\le k\right\}$ denotes the set of $n \times n$ symmetric matrices with rank not exceeding $k$ and $\langle \cdot, \cdot\rangle$ denotes the matrix inner product, i.e., $\langle \mathbf{A}, \mathbf{B} \rangle = tr(\mathbf{A}^T \mathbf{B})$. Instead of minimizing the Frobenius norm of the reconstruction error (for the covariance matrix) in PCA, PPCA minimizes the distributionally-robust bound on expected optimality loss regularized with $\langle \mathbf{\Sigma}_1, \mathbf{\Sigma}_e \rangle$, which is a relaxation of the requirement on the columns of $\mathbf{\Sigma}_1$ and $\mathbf{\Sigma}_e$ be orthogonal (i.e., $\langle \mathbf{\Sigma}_1, \mathbf{\Sigma}_e \rangle=0$). The first two constraints in the above formulation require the second moments of both $\bmt{z}_1$ and $\bmt{z}_e$ to be valid, i.e., there exist valid multivariate distributions with the corresponding covariance matrices that sum up to $\mathbf{\Sigma}_0$. 
The constraint \eqref{eq:rank} is written in a generic form such that it could enforce any desired structural properties on the projection, e.g., low rank and sparsity. 

The subproblem \eqref{eq:persistency} to compute $\bar{h}(\mathbf{0}, \mathbf{\Sigma}_e)$ is in the maximization form. Thus, problem \eqref{eq:min-max} is a min-max problem; in fact, it can be interpreted as a robust optimization problem. 
In the literature, min-max robust optimization formulations are typically reformulated as minimization problems using the duality of the inner problem. Following this standard approach, we have the following result. 

\begin{proposition} \label{prop:PPCA_SDP}
The PPCA problem can be reformulated as:
\begin{eqnarray}
		\min \quad &&\boldsymbol{\alpha}^T \mathbf{b} + \boldsymbol{\beta}^T \mathbf{b}^2 + g_1 +  \langle 
\mathbf{\Lambda}+\theta \mathbf{\Sigma}_1, \mathbf{\Sigma}_e \rangle\label{eq:min-min} \\
\mbox{s.t.}		&& \eqref{eq:Sigma}-\eqref{eq:rank}, \nonumber \\
		&& \mathbf{G} = \begin{pmatrix} 
		g_1 &   \boldsymbol{g}_2^T & \mathbf g_3^T\\
		\boldsymbol{g}_2 & \mathbf{G}_{22}  & \mathbf G_{32}^T \\
		\mathbf g_3 & \mathbf G_{32} & \mathbf G_{33}
	\end{pmatrix} \succeq 0 \label{eq:G} \\	
		&& \begin{array} {r@{}ll}
		& g_1 = \nu  &   \quad \mathbf{G}_{22} = \mathbf{\Lambda} \\
		& \mathbf{g}_2 = \frac{1}{2}\boldsymbol{\gamma} & \quad \mathbf{G}_{32} = - \frac{1}{2}\mathbf{I} \\
		&   \mathbf{g}_3 \le \frac{1}{2}\mathbf{A} \boldsymbol{\alpha} & \quad \mathbf{G}_{33} = \sum_{i=1}^m \beta_i \mathbf{a}_i \mathbf{a}_i^T,
	\end{array} \label{eq:dual-constraints}
\end{eqnarray}
\noindent provided there exists a feasible solution satisfying \eqref{eq:G} with strict positive definiteness.  
\end{proposition}

The problem \eqref{eq:min-min} is not a convex optimization problem and is difficult to solve in general. 
In particular, the objective \eqref{eq:min-min} is non-convex due to the bilinear inner product $\langle 
\mathbf{\Lambda}+\theta \mathbf{\Sigma}_1, \mathbf{\Sigma}_e \rangle $. Furthermore, fixing any feasible value of $\mathbf{\Lambda}+\theta \mathbf{\Sigma}_1$ and solving the remaining problem yields an upper bound on the original problem. This suggests the alternating algorithm~\ref{alg:alternating}.

\begin{algorithm}[tb]
   \caption{Alternating Algorithm}
   \label{alg:alternating}
\begin{algorithmic}
   \STATE {\bfseries Input:} problem parameters $\mathbf{A}$, $\mathbf{b}$, $\boldsymbol{\Sigma}_0$, and penalty $\theta$
   \STATE Initialize $\hat{\boldsymbol{\Sigma}}_e$ such that $\hat{\boldsymbol{\Sigma}}_e \succeq 0$ and $\hat{\boldsymbol{\Sigma}}_1 = \boldsymbol{\Sigma}_0 - \hat{\boldsymbol{\Sigma}}_e \succeq 0$.
   \REPEAT
   \STATE{1. Given $\mathbf{\Sigma}_e = \hat{\boldsymbol{\Sigma}}_e$ and $\mathbf{\Sigma}_1 = \hat{\boldsymbol{\Sigma}}_1$, solve problem \eqref{eq:min-min} over $(\boldsymbol{\nu}, \boldsymbol{\gamma}, \alpha, \boldsymbol{\beta}, \boldsymbol{\rho}, \mathbf{\Lambda}, \mathbf G)$. Save the optimal value of $\boldsymbol{\Lambda}$ as $\hat{\boldsymbol{\Lambda}}$.} 
   \STATE{2. Given $\mathbf{\Lambda}+\theta \mathbf{\Sigma}_1=\hat{\mathbf{\Lambda}}+\theta \hat{\boldsymbol{\Sigma}}_1$, solve problem \eqref{eq:min-min} over $(\boldsymbol{\nu}, \boldsymbol{\gamma}, \alpha, \boldsymbol{\beta}, \boldsymbol{\rho}, \boldsymbol \Lambda, \mathbf{G}, \mathbf{\Sigma}_1, \mathbf{\Sigma}_e)$. Save the optimal value of $\mathbf{\Sigma}_e$ as $\hat{\boldsymbol{\Sigma}}_e$ and $\mathbf{\Sigma}_1$ as $\hat{\boldsymbol{\Sigma}}_1$.}
   \UNTIL{no improvement in the objective value}
   \STATE {\bfseries Output:} $\mathbf{\Sigma}_1$
\end{algorithmic}
\end{algorithm}

It is easy to see that, the objective value (weakly) improves in every iteration of the algorithm. Therefore, though the algorithm may not necessarily converge to the global optimal solution, it will converge to a local minimum. 
We also remark that it is possible to initialize the algorithm with the PCA solution. Thus the algorithm, even if terminated early, can guarantee to produce a solution better than PCA in terms of the worst-case expected performance bound. 

We further remark that it is straightforward to include any convex regularization term in problem \eqref{eq:min-min}. In particular, penalizing the Frobenius norm of the reconstruction error $\boldsymbol{\Sigma}_e$ works well in our computational studies. For example, \eqref{eq:min-min} can be replaced with:
\begin{equation*}
\min \eta \left(\boldsymbol{\alpha}^T \mathbf{b} + \boldsymbol{\beta}^T \mathbf{b}^2 + g_1 +  \langle 
\mathbf{\Lambda}+\theta \mathbf{\Sigma}_1, \mathbf{\Sigma}_e \rangle\right) +  (1 - \eta) \| \boldsymbol{\Sigma}_e \|_F + \rho\|\Sigma_1 \|_{*} , \label{eq:min-min-reg}
\end{equation*}
\noindent for some $\eta \in [0, 1]$. The cases with $\eta = 0$ and $\eta=1$ reduce to the conventional PCA and the unregularized PPCA \eqref{eq:min-min}, respectively. Our computational experiments suggest that a lightly regularized objective (e.g., $\eta = 0.95$) tends to work better than the unregularized version.

\subsection{Tight Distributionally Robust Bound}
Algorithm \ref{alg:alternating} yields a low-rank covariance matrix and computes an upper bound on the worst-case error in approximating the expected second-stage objective value. In general, this bound is not necessarily tight for three reasons. First, the performance bound \eqref{eq:bound_error} makes use of a subadditivity relation and is not tight in general. Second, the distributionally-robust formulation \eqref{eq:persistency} is not exactly tight, in that there is no guarantee that a feasible distribution (or a sequence thereof) exactly (or asymptotically) achieves this bound. Third, the alternative algorithm converges to a local minimum to problem \eqref{eq:min-min} which is not necessarily globally optimal. 

Our computational experiments show that the algorithm is effective in identifying low-rank data projections that perform well in downstream optimization problems in practice. Yet, it is of theoretical interest to close these three gaps. Below, we provide a formulation that computes a tight upper bound on the worst-case expected approximation error given a projection vector $\mathbf{V}$. 

Let the projected polyhedron be $\hat{\boldsymbol{P}} = \{ \hat{\mathbf{w}} | \hat{\mathbf{w}}=  \mathbf{V} \mathbf{V}^T \mathbf{w},  \mathbf{w} \in \boldsymbol{P} \}$. For a data point $\bmt{z}$ and a projection $\mathbf{V} \mathbf{V}^T \bmt{z}$, the approximation error is:
\begin{align*}
	H(\bmt{z})=h(\mathbf{V} \mathbf{V}^T \bmt{z})-h(\bmt{z})  =\max_{\hat{\mathbf{w}} \in \hat{\mathbf{P}}} \quad & \hat{\mathbf{w}}^T \bmt{z} - \max_{\mathbf{w} \in \mathbf{P}} \mathbf{w}^T \bmt{z}	\\
	 =\max_{\mathbf{w,y,s}} \quad & \mathbf{V} \mathbf{V}^T \mathbf{w} \bmt{z} -\mathbf{b}^T \mathbf{y}\\
	 \mbox{ s.t. } \quad &\mathbf{A} \mathbf{y} -\mathbf s =   \bmt{z}\\
	 & \mathbf{A}^T \mathbf{w} = \mathbf b\\
	 & \mathbf{w,s} \ge 0.
\end{align*}

Given the mean $\boldsymbol{\mu}$ and covariance matrix $\boldsymbol{\Sigma}$ of $\bmt{z}$ respectively, the tight distributionally robust bound on the mean absolute error (MAE) can be evaluated as:
\begin{align}
	Z_P=\sup_{\bmt z \sim (\bm \mu, \bm \Sigma)} \left|\mathbb{E} [H(\bmt{z})]\right|.
\end{align}

\begin{theorem} \label{thm:exact}
The distributionally robust bound on the MAE in low-rank approximation with $\mathbf{V}$ is given by $Z_P=\max\{Z_C^+,Z_C^-\}$, where $Z_C^+$ and $Z_C^-$ can be evaluated by the following convex optimization problems:
\begin{eqnarray*}
	\begin{array}{ll}
		Z_C^+=\max & \langle \mathbf{VV}^T, \mathbf Y_w\rangle-\mathbf b^T \mathbf p_y \\
		\qquad \text { s.t. } & \mathbf a_i^T \mathbf p_{w}=b_i, \quad \forall i=1,...,m \\
		& \mathbf a_i^T \mathbf X_{w} \mathbf a_{i}=b_{i}^{2}, \quad \forall i=1,...,m \\
		& \langle \mathbf A^T \mathbf A, \mathbf X_{y}\rangle=\langle \mathbf A^T, \mathbf Y_{y}+\mathbf Z_{sy}^T\rangle \\
		& \langle \mathbf A^T \mathbf A, \mathbf X_{y}\rangle=\langle \mathbf I, \mathbf X_{s}+2 \mathbf Y_{s}+\bm \Sigma\rangle\\
		& \left(\begin{array}{ccccc}
			1 & \bm \mu^{T} & \mathbf p_w^{T} & \mathbf p_y^{T} & \mathbf p_s^{T} \\
			\bm \mu & \bm \Sigma & \mathbf Y_w^{T} & \mathbf Y_y^T & \mathbf Y_s^T\\
			\mathbf p_w & \mathbf Y_w & \mathbf X_w & \mathbf Z_{yw}^T & \mathbf Z_{sw}^T\\
			\mathbf p_y & \mathbf Y_y & \mathbf Z_{yw} & \mathbf X_y & \mathbf Z_{sy}^T \\
			\mathbf p_s & \mathbf Y_s & \mathbf Z_{sw} & \mathbf Z_{sy} & \mathbf X_s \\
		\end{array}\right) \in 
  \left\{\mathbf{M} \left| \exists \begin{array}{l}
			V_{1} \in \mathbb{R}_{+}^{1 \times l} \\
			\mathbf{V}_{2} \in \mathbb{R}^{n \times l} \\
			\mathbf{V}_{3} \in \mathbb{R}_{+}^{n \times l}\\
			\mathbf{V}_{4} \in \mathbb{R}^{m \times l}\\
			\mathbf{V}_{5} \in \mathbb{R}_{+}^{n \times l}
		\end{array}\right.\right.  \text { s.t. } \left. \mathbf{M}=\left(\begin{array}{l}
			V_{1} \\
			\mathbf{V}_{2} \\
			\mathbf{V}_{3} \\
			\mathbf{V}_{4} \\
			\mathbf{V}_{5}
		\end{array}\right)\left(\begin{array}{l}
			V_{1} \\
			\mathbf{V}_{2} \\
			\mathbf{V}_{3} \\
			\mathbf{V}_{4} \\
			\mathbf{V}_{5}
		\end{array}\right)^{T}\right\},
	\end{array}\label{eq:Z_C}
\end{eqnarray*}
and 
\begin{eqnarray*}
	\begin{array}{ll}
		Z_C^-=\max & \mathbf b^T \mathbf p_y - \langle \mathbf{VV}^T, \mathbf Y_w\rangle\\
		\qquad \text { s.t. } & \mbox{Constraints in } Z_C^+.
	\end{array}
\end{eqnarray*}
\end{theorem}

\section{Computational Experiments with Synthetic Data}
We first apply the PPCA approach to a stochastic programming problem with a set of synthetic, simulated data. This allows for the evaluation of the effectiveness of our proposed approach under a controlled setting. In the next section, we further illustrate its application based on real-life data.

\subsection{Problem Setting}\label{sec: simu_setting}
We consider a joint production and inventory allocation problem with transshipment in a network consisting of a set of demand nodes $I$, and a set of production nodes $I'$. Each production node $i \in I'$ has a production capacity of $S_i$, and each demand node $j \in I$ faces stochastic demand $\tilde{z}_j$ (where $\bmt{z} \in \mathbb{R}^{|I|}$ denotes the vector of demand).  The problem consists of two stages. In the first stage, the firm determines how much to produce in each site $i \in I'$ (with unit production cost $f_i$), as well as the shipment quantity $x_{ij}$ from each $i \in I'$ to each demand location $j \in I$ (with unit shipment cost $c_{ij}$. Then, in the second stage, demand is realized, and the firm can fulfill demand with on-hand inventory with possible transshipments: in particular, it determines the transshipment quantity $y_{ij}$ for demand nodes $i, j \in I$ (at unit transshipment cost $\bar c_{ij}$). Unmet demand at $i \in I$, denoted by $w_i$, will be penalized with a unit shortage cost of $p_i$. This problem can be formulated as a two-stage stochastic program as below:
\begin{eqnarray}
	\min_{\mathbf{X} \ge 0} &&  \sum_{i\in I',j\in I} (f_i+c_{ij})x_{ij}+ \mathbb E [h(\mathbf{X,\tilde z})] \label{eq:experiment} \\
	\mbox{s.t.} && \sum_{j\in I} x_{ij} \le S_i, \forall i\in I', \nonumber
\end{eqnarray}
\noindent where $\mathbf{X} = (x_{ij}, \forall i \in I', j \in I)$ and $h(\cdot)$ denotes the second stage cost, given by:
\begin{align*}
	h(\mathbf{X, z}) = &\min_{\mathbf{Y,w}\ge 0} \sum_{i\in I,j\in I} \bar c_{ij}y_{ij} + \sum_{i \in I} w_i p_i \\
	\mbox{s.t.} \quad & -\sum_{j\in I} y_{ij}+\sum_{j\in I} y_{ji}+w_i \ge  z_i-\sum_{j\in I'} x_{ji}, \forall i \in I.
\end{align*}
Here, $ z_i$ is the realized demand at node $i$ and $\mathbf p=(p_i,\forall i\in I)$ is the vector of penalty costs. Note that this problem has complete recourse. 

\subsection{Synthetic Data Generation}
We generate a synthetic data set based on the 49-node problem instance for facility location problems from \cite{daskin1995network}, where the 49 demand nodes (set $I$) are the 48 continental U.S. state capitals and Washington, DC., and the shipping costs ($c_{ij}$ and $\bar c_{ij}$) are proportional to the great circle distances between any pair of locations $i$ and $j$. We generate stochastic demand at node $i$ as follows based on the notion of primitive uncertainties \cite{chen2008linear}:
\[\tilde z_i=\phi_i(\xi_{i1}\tilde \zeta_1+\xi_{i2}\tilde \zeta_2+\cdots+\xi_{iK}\tilde \zeta_K)^+,\]
where $(\zeta_1,\cdots,\zeta_K)$ are the primitive uncertainties following some joint distribution (e.g., independent Gaussian distributions), $\xi_{ik}$ denotes fixed coefficients sampled from Uniform($-0.8,1$), and $\psi_i> 0$ is a scaling vector proportional to the corresponding demand nodes $i$' population. Moreover, the operator $\times$ denotes elementwise multiplication and $(\cdot)^+$ denotes $\max(0,\cdot)$. In our experiments, we set $K=25$ and $\mathcal F$ to be componentwise independent, which implies that the covariance matrix of $\bmt{z}$ has a rank up to 25. We then sample 100 and 1000 observations as the training and test sets in each experiment instance. 

Furthermore, to evaluate the potential impact of data perturbation or contamination common in real-life applications, we run a set of experiments where the training data includes a random noise $\varepsilon_i$ following a Gaussian distribution. In this case, the training data is generated as follows:
\[\tilde z_i=\phi_i(\xi_{i1}\tilde \zeta_1+\xi_{i2}\tilde \zeta_2+\cdots+\xi_{iK}\tilde \zeta_K+\varepsilon_i)^+.\]

We also randomly select five of the 49 nodes as the production sites (set $I'$), each with production capacity $S_i$ set to $40\%$ of the sum of the mean demand over the 49 nodes.  The production cost $f_i$ at site $i \in I'$ is sampled from the Uniform($10,20$) distribution. The shipping cost in the first stage $c_{ij}=0.015 \times$ the distance from node $i$ to $j$; and the transshipment cost in the second stage is set to $\bar c_{ij}=0.02\times$ the distance from node $i$ to $j$. Finally, the penalty cost per unit of lost sales is $p_i=100$ for all $i\in I$.

\subsection{Performance Evaluation}

We first obtain a low-dimensional representation by solving the (regularized) PPCA problem with Algorithm \ref{alg:alternating}. This yields a low-rank covariance matrix $\boldsymbol{\Sigma}_1$ that approximates $\boldsymbol{\Sigma}_0$, estimated from the training sample after centering the data. By re-solving the problem with varying weights $\rho$ on the nuclear norm regularization term, we obtain $\boldsymbol{\Sigma}_1$ with different ranks (values of $k$). For each $\Sigma_1$, we can recover the associated low-dimensional data projection $\mathbf{V}$ via eigenvalue decomposition.

Given $\mathbf{V}$, the $\mathbf{V} \mathbf{V}^T \bmt{z}$ gives the projection of the $n$-dimensional random vector $\bmt{z}$ onto a $k$-dimensional subspace of $\mathbb{R}_n$. We test the performance of approximating the stochastic program based on the alternative low-dimensional projections identified with PPCA (our proposed approach) and PCA (as a benchmark) with the same $k$. In particular, we obtain first-stage production and shipping decisions ($\mathbf{X}$) by solving the LDR-based approximation for stochastic programs discussed in \cite{chen2008linear}. Importantly, projecting the demand vector onto a $k$-dimensional subspace implies modeling the $n$-dimensional demand vector with $k$ features, i.e., the (prescriptive) principal components. Following the LDR approach in \cite{chen2008linear}, we restrict the recourse decisions (transshipments, $\mathbf{y}$) to be affine functions of the principal components (i.e., with $k+1$ degrees of freedom), instead of the original demand ($n+1$ degrees of freedom).  Thus, dimensionality reduction effectively reduces the complexity of the LDR formulation. For example, as illustrated in Figure~\ref{fig: solve_time}, the computational times can be reduced by as much as 98.7\% when the LDR is defined based on two-dimensional primitive uncertainties (identified with PPCA) than the original 49-dimensional demand vector. 

\begin{figure}[htbp]
\begin{center}
\centerline{\includegraphics[width=0.6\columnwidth]{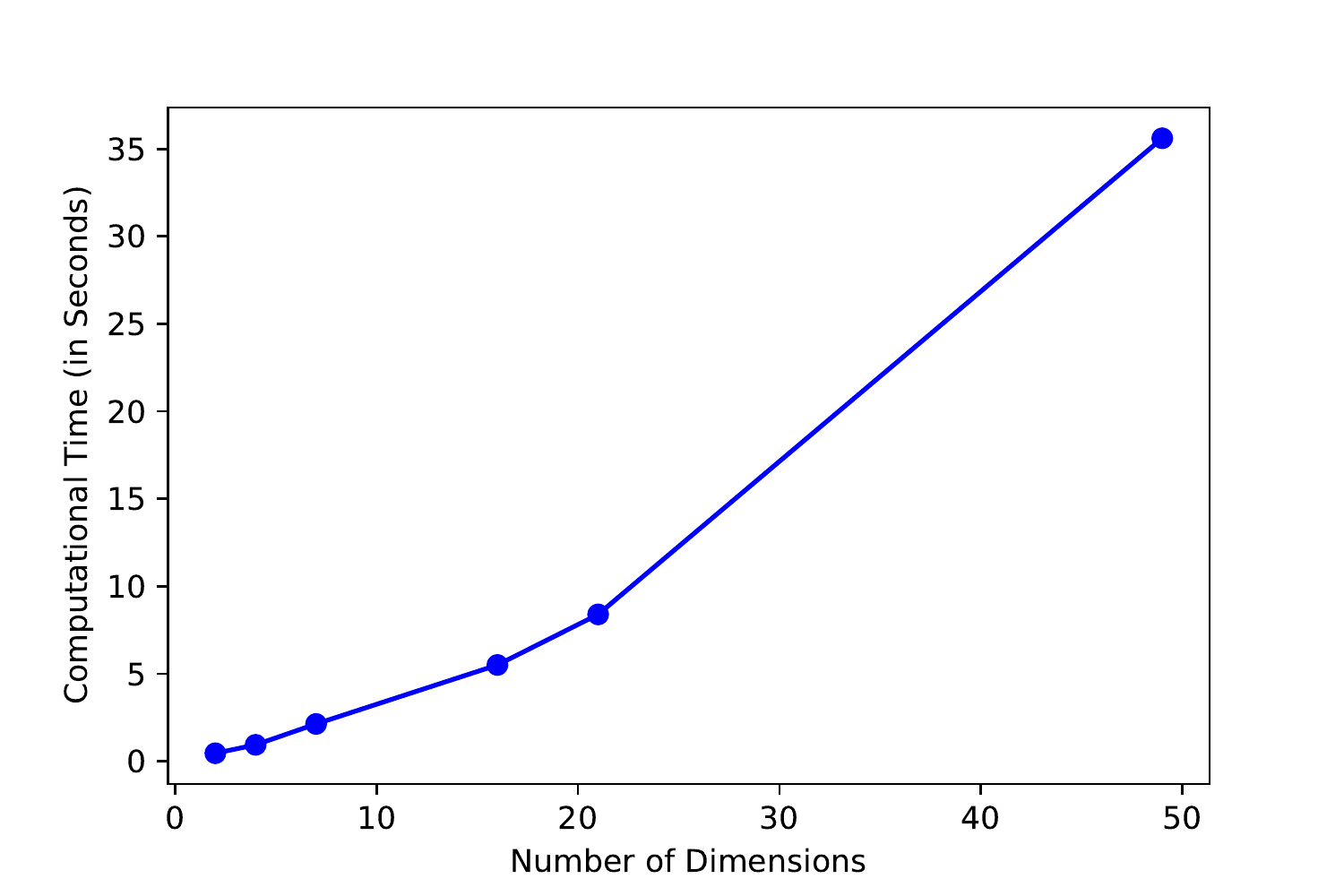}}
\caption{Computational time for solving stochastic program using LDR against the number of dimensions ($k$) when $\zeta_i\sim$Normal(2,1).} \label{fig: solve_time}
\end{center}
\vskip -0.2in
\end{figure}

To evaluate the performances of the low-rank projections, we evaluate the first-stage decisions  obtained by solving the resulting LDR formulations \cite{chen2008linear}, via sample average approximation (SAA) approach with the test data. For each solution, we compute the optimality gap, i.e., the relative difference from the ``true'' optimal cost assuming knowledge of the test data (evaluated via SAA).
We first report the experiment where the primitive uncertainties $\tilde \zeta_k$'s are independent and identically distributed following univariate normal distributions, and the training data is sampled without noise. Figure~\ref{fig:normal} shows the optimality gaps using PPCA and PCA at varying values of $k$. We find that PPCA is very effective in identifying low-dimensional representations for the stochastic program. In particular, it allows for projecting the 49-dimensional demand data onto subspaces of below 5 dimensions ($k<5$), while maintaining very small (2-3\%) optimality gaps in all cases. For regular PCA, projecting demand data onto $k=5$ or below can lead to poor performance (e.g., optimality gap exceeding 20\%). The performance of PCA is only able to match PPCA when $k$ is sufficiently large ($k\ge 10$), i.e., leading to less efficient representations of the data. 
\begin{figure}[htbp]
\begin{center}
\centerline{\includegraphics[width=0.6\columnwidth]{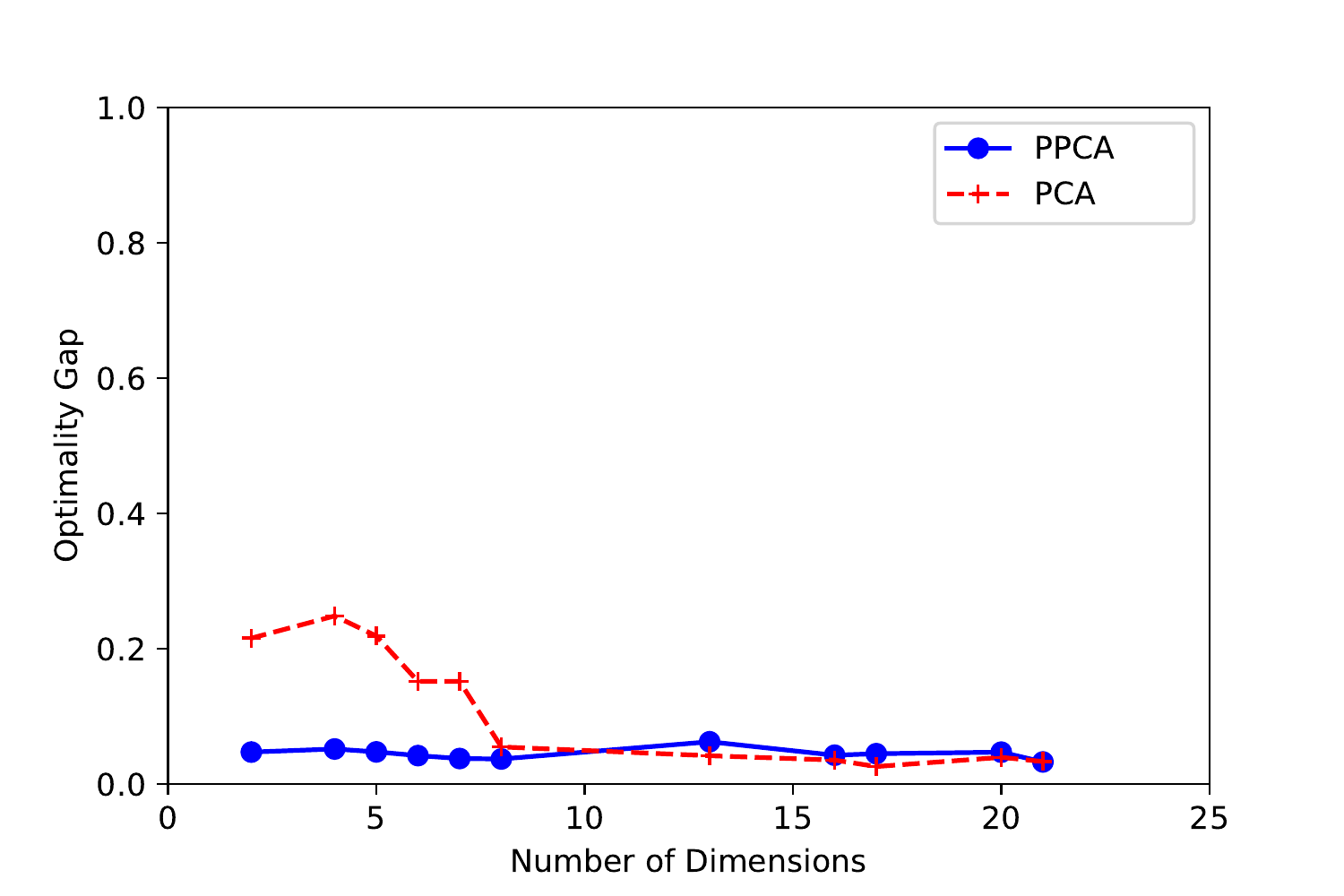}}
\caption{$\zeta_i  \sim$Normal(2,1) without noise.}\label{fig:normal}
\end{center}
\vskip -0.2in
\end{figure}

A prime motivation for dimensionality reduction is to filter out noises in training data. In practice, the training samples for many stochastic optimization problems may be subject to noises, e.g., from data collection or demand censoring. To test the effectiveness of PPCA under noisy data, we consider the noise terms $\varepsilon_i$'s to be drawn independently from a univariate Normal distribution with zero mean and a standard deviation that is proportional to the mean demand. We repeat the computational experiment given this set of noisy training data and compare the performances of PPCA and regular PCA. Compared with Figure~\ref{fig:normal}, in Figure~\ref{fig:normal_noise}, we see that the performance of prescriptive PCA is robust with respect to noise; however, PCA performs very poorly under noisy training data, even at high values of $k$. Furthermore, the performance of low-dimensional representation deteriorates more significantly when the variances of random demands are smaller. It is because, as the true variances of demands decrease, a larger proportion of the variation in the training data comes from noise. Importantly, we find that the performance of PCA deteriorates substantially more than PPCA, showcasing the robustness of PPCA under noisy data. Unlike the case of noise-free data (Figure \ref{fig:normal}), PCA no longer achieves similarly low optimality gaps than does PPCA unless $k$ is very high ($k \ge 20$; recall that the rank of the true demand distribution is 25). 
\begin{figure}[htbp]
\begin{center}
\centerline{\includegraphics[width=0.6\columnwidth]{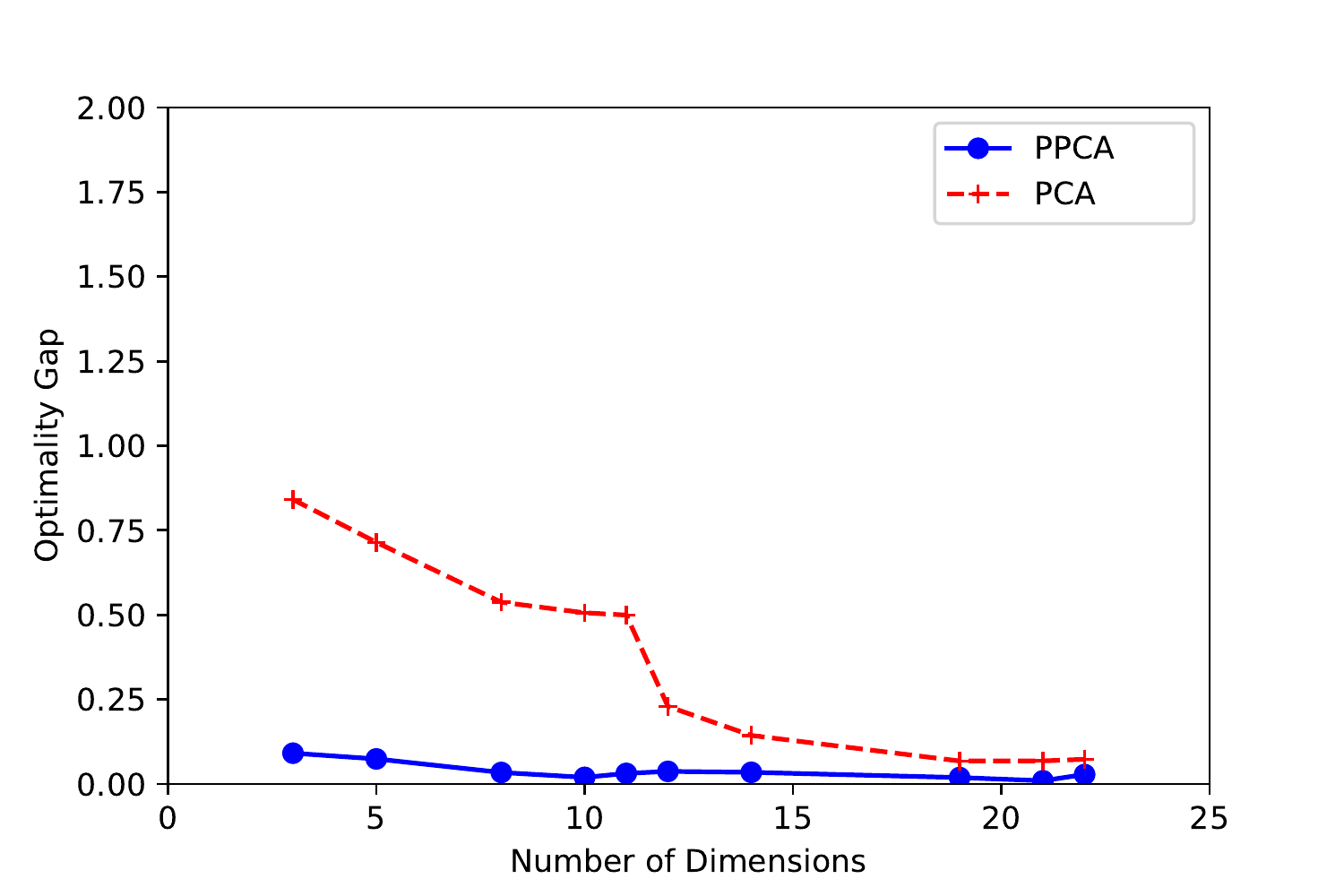}}
\caption{$\zeta_i \sim$Normal(2,1) with noise.}\label{fig:normal_noise}
\end{center}
\vskip -0.2in
\end{figure}

\section{Case Study: NYC Taxi Pre-Allocation}
Having illustrated the effectiveness of PPCA with a synthetic data set, we further examine its performance with the New York City taxi data set \cite{NYCTaxi}. We consider a mobility platform, e.g., a ridesharing platform with autonomous vehicles, serving the 59 taxi zones on Manhattan island and assume that the platform faces passenger demand as recorded in the taxi trip data. To ensure high availability of cars in close vicinity of passengers, the platform needs to reposition idle cars to prospective demand (pick-up) locations ahead of demand realization especially during the morning or afternoon peak hours when demand is highest. 

Following \cite{hao2020robust} that considers the vehicle pre-allocation problem for a single bottleneck demand period, we formulate this problem as a two-stage stochastic program: In the first stage (before the peak hour), the platform reposition idle cars at certain (location-dependent) costs; Then in the second stage, peak-hour trip demand is realized, and the platform has to match realized passenger demand with the available cars across the city, to minimize travel distances for cars or waiting times for passengers. Considering a risk-neutral objective, this problem can be expressed as a special case of problem \eqref{eq:experiment} in Section~\ref{sec: simu_setting} by re-interpreting the notation as follows. The set $I = I'$ is the set of taxi zones, where each zone $i$ is endowed with $S_i$ vehicles at the beginning of the first stage and faces random demand $\tilde{z}_i$ to be realized in the second stage. The first-stage decisions involve repositioning $x_{ij}$ vehicles between each pair of zones $i, j \in I$, at a unit cost $c_{ij}$ per vehicle. For each unit of realized demand at zone $j \in I$ in the second stage, a matching cost (e.g., customer waiting cost) of $c_{ij}$ is incurred if it is met by a vehicle positioned at zone $i$, and a penalty cost $w_j$ is incurred if it is unmet. Unlike the inventory transshipment setting in Section 6, there is no production cost (i.e., $f_i$ = 0).

We use the trip records of Yellow Taxis in Manhattan from 8:00am -- 8:59am daily between June to August 2020 provided by \cite{NYCTaxi}. 
For each day, we count the number of trips in each taxi zone, resulting in 92 observations of a 59-dimensional random demand vector. Moreover, the costs of repositioning cars (in the first stage) and traveling to meet customer demand (in the second) between zones are proportional to the average trip fare between those zones observed in the data. The penalty cost for unmet demand in each zone is estimated based on the average fare for all trips originating from that zone. Finally, we assume the initial total supply of cars to be equal to the mean demand, and randomly distributed in 10 randomly selected zones. We then evenly split the 92 observations of trip demand into training and test sets. 

\begin{figure}[htbp]
	\centering
	\includegraphics[width=0.6\textwidth]{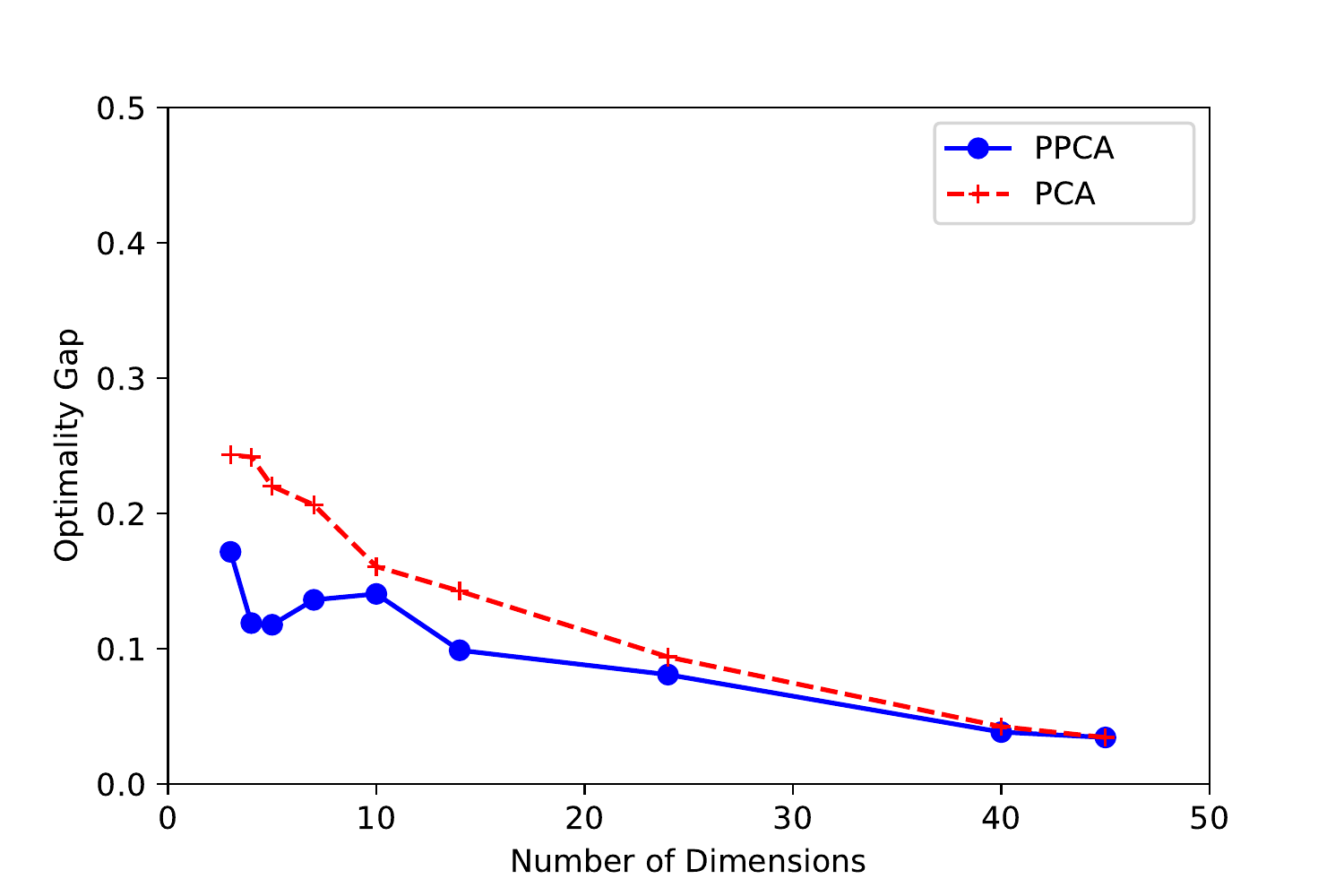}
	\caption{The optimality gaps using PCA and PPCA for NYC Taxi Pre-allocation.}\label{fig:nyc_gaps}
\end{figure}

Following similar procedures as in the previous section, we solve for the first-stage vehicle repositioning decisions in the low-dimensional subspace identified by PPCA and PCA respectively based on the training data. We then evaluate the out-of-sample performance of the solution with test data. The optimality gaps are shown in Figure~\ref{fig:nyc_gaps}. 
First, we observe that both PPCA and regular PCA achieve lower optimality gaps as $k$ increases, as expected. In addition, we can see that the optimality gaps tend to be higher than observed for synthetic data (without noise) across the range of $k$. This is inevitable, as the assumption that training and test data are identical and independently distributed only holds approximately in real-life data. Thus, the case with real-life data is more comparable with the case of synthetic data with (some) noise in the training data. 
 
Similar to the case with synthetic data, we find that the performance of PPCA dominates that of regular PCA: the former is able to achieve lower optimality gaps (better solution quality) with equal or lower values of $k$ (i.e., the computational burden in solving the stochastic program). For example, the performances of PPCA with $k < 10$ is similar to that of PCA with $k \approx 20$. This reaffirms the insight that PPCA projects the data along dimensions that retains more relevant information with respect to the stochastic program than does regular PCA. 

Moreover, we visualize the top two principal components (PCs) identified by PPCA and regular PCA in Figure \ref{fig:Top_PCs}. The first PCs, as identified by both methods, are almost identical. The correlation between the two sets of loadings is over 0.99. This indicates that both methods are consistent in finding the first dominant factor of variation in the data. However, the second PCs differ significantly. From PCA, the second PC highlights an axis of strong demand variation, where zones 42 (Upper Manhattan/ Harlem) and 43 (Central Park) exhibit clear negative correlation and all other zones carry relatively uniform weights. From PPCA, however, zone 43 not only exhibits a clear negative correlation with zone 42, but also a cluster of zones in Midtown. In our stochastic program, such a factor highlights the importance of moving vehicles between the Central Park and Midtown areas in the recourse problem. While this pattern does not necessarily carry the largest variation as opposed to the PCA solution, it contains relevant information for the optimization problem that the PCA does not capture. 

\begin{figure}[htbp]
	\begin{subfigure}{0.5\linewidth}
		\centering
		\includegraphics[width=\linewidth]{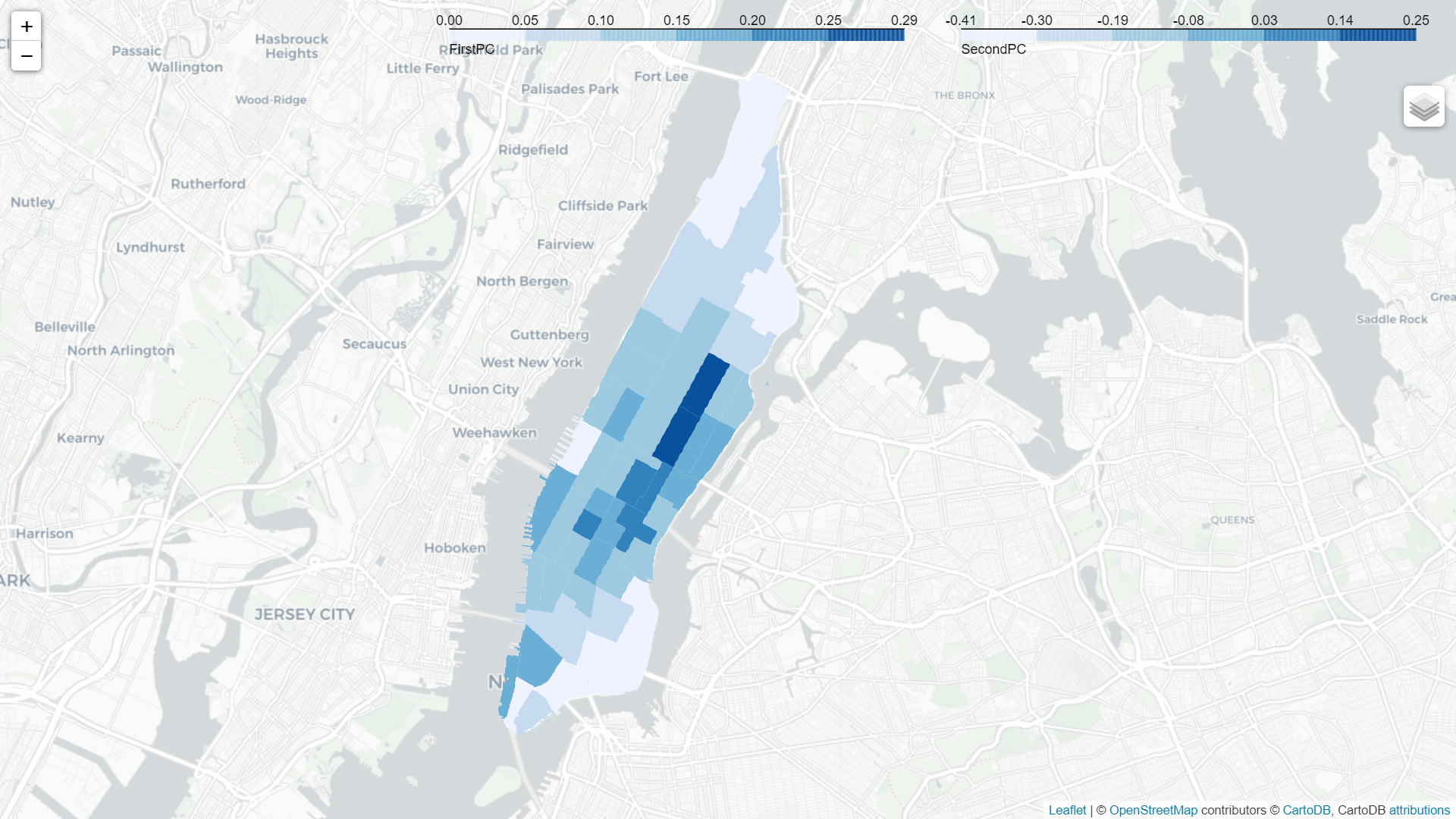}
		\caption{First PC (PPCA).}
	\end{subfigure}
        \hspace{0.2in}
	\begin{subfigure}{0.5\linewidth}
		\centering
		\includegraphics[width=\linewidth]{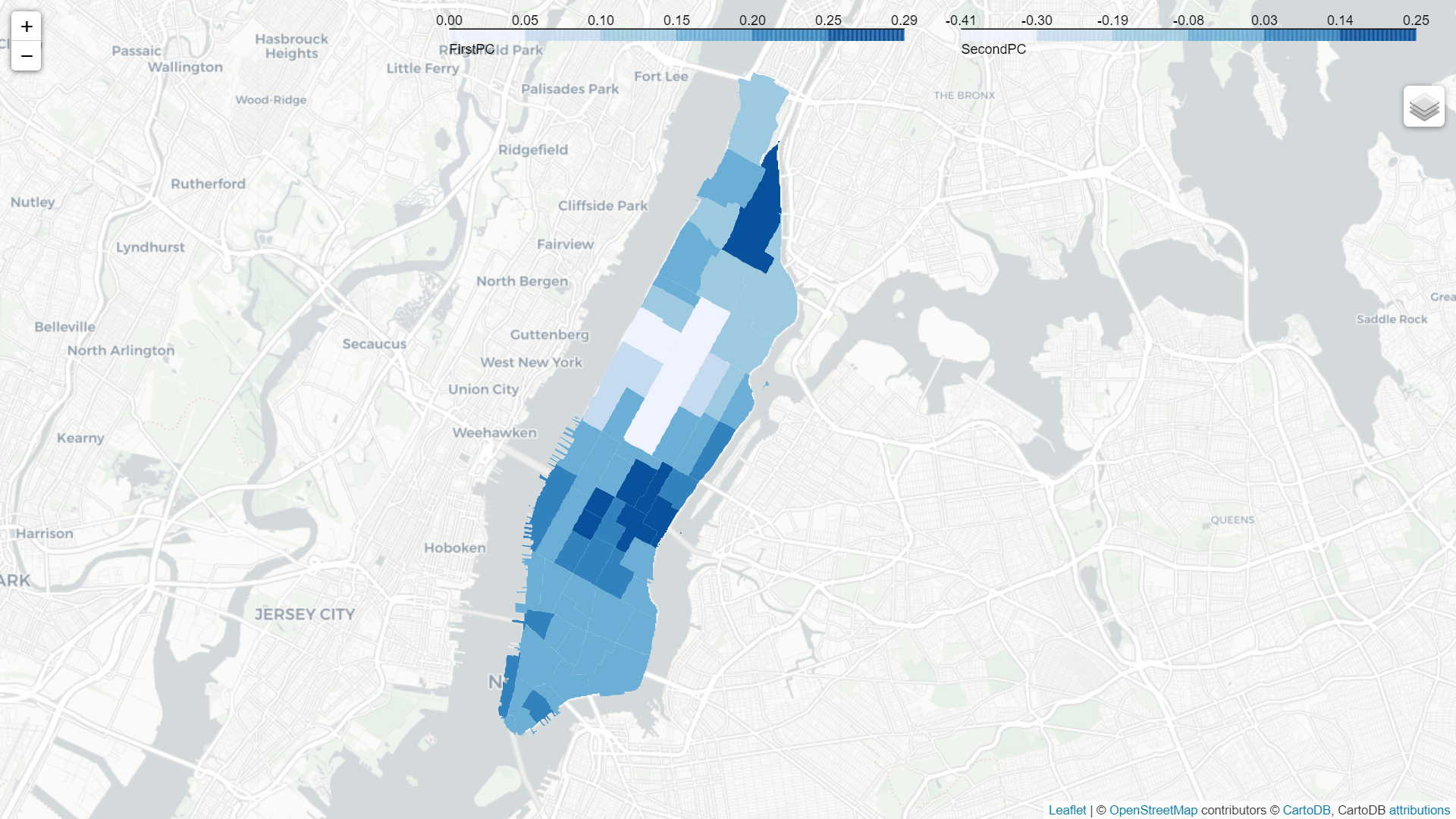}
		\caption{Second PC (PPCA)}
	\end{subfigure}
		\begin{subfigure}{0.5\linewidth}
		\centering
		\includegraphics[width=\linewidth]{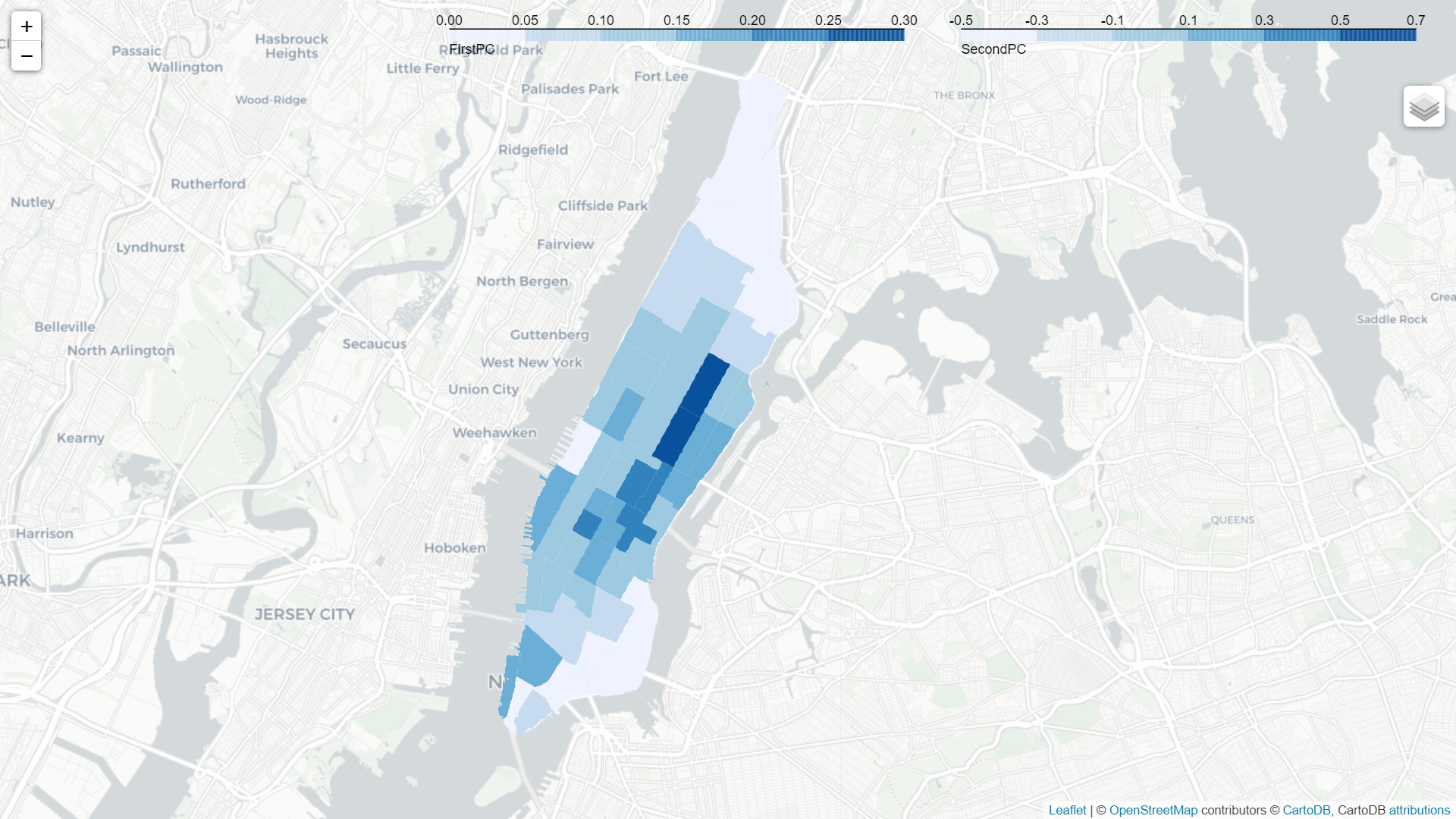}
		\caption{First PC (PCA)}
	\end{subfigure}
        \hspace{0.2in}
	\begin{subfigure}{0.5\linewidth}
		\centering
		\includegraphics[width=\linewidth]{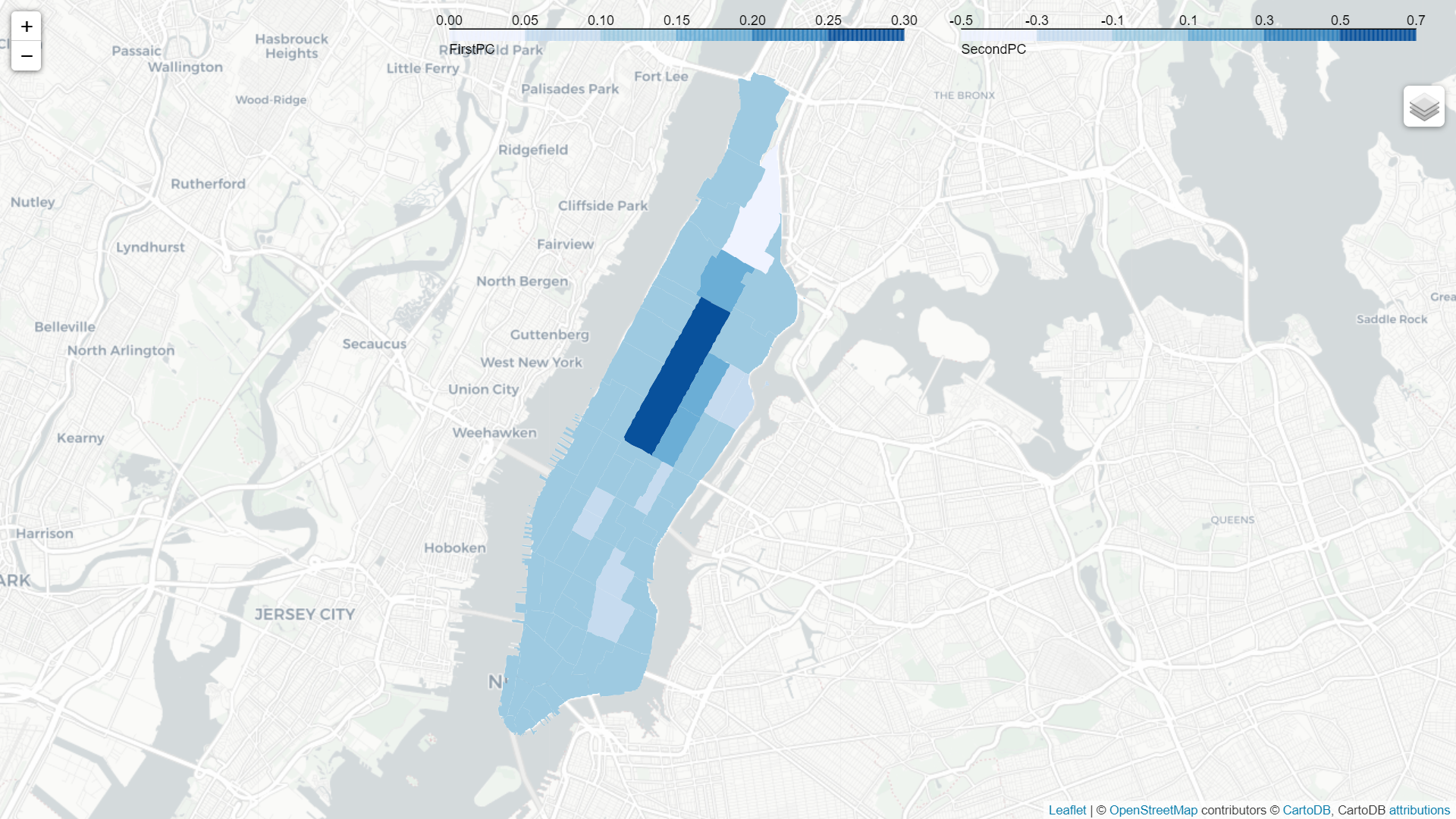}
		\caption{Second PC (PCA)}
	\end{subfigure}
	\caption{The top two principal components (PCs) from PPCA and PCA. }\label{fig:Top_PCs}
\end{figure}

\section{Conclusion} 
As the integration of data science and machine learning tools with mathematical programming becomes more prevalent in practice, there is a stronger need for the development of methods to align these tools. Our paper contributes to this growing literature by proposing a prescriptive dimensionality reduction method that uncovers a low-dimensional linear projection of data preserving optimality in the sense of solving a downstream stochastic program, as opposed to preserving variation in the data as in standard approaches. 

Our work can be extended in several directions. First, while our analysis has focused on rank reduction of the covariance matrix, i.e., the case of PCA, a similar approach can be adopted for different dimensionality reduction settings such as sparse PCA and factor analysis. This would involve imposing the applicable structural constraint in the optimization problem \eqref{eq:min-max}, and the main challenge would lie in developing reformulations and/or approximations amenable to computationally-efficient solution algorithms.

Second, while our focus was on the unsupervised learning context of linear dimensionality reduction in this paper, it is possible to follow a similar approach for supervised learning tasks. A direct extension of our approach could be applied to the case of low-rank linear regression. Specifically, instead of seeking a low-rank projection of given data points, one would seek a low-rank regression model that predicts the dependent variable to be used in a downstream optimization problem given input features. It would be interesting to explore similar directions with other machine learning methods, such as ensemble methods (e.g., random forests).


\bibliographystyle{unsrt}  
\bibliography{dim_red}

\newpage
\appendix

\section{Proofs }

\subsection{Proof of Proposition \ref{prop:PPCA_SDP}}
We can rewrite the inner maximization problem as follows (with dual variables associated with each constraint in parentheses):
\begin{equation*}
	\begin{array}{r@{}l}
		\bar{h}(\mathbf{0}, \mathbf{\Sigma}) =   
		\max_{u, \mathbf{v}, \mathbf{p},\mathbf Y, \mathbf X} \quad &tr(\mathbf Y)\\
		\text{s.t.} \quad & \mathbf{a}_i^T \mathbf{p}=b_i, \quad \mbox{ for } i=1, \cdots, m \quad (\alpha_i) \\
		&\mathbf{a}_i^T \mathbf X \mathbf{a}_i = b_i^2, \quad \mbox{ for } i=1, \cdots, m \quad (\beta_i) \\
		& u = 1 \quad (\nu) \\
		& v_i = 0, \quad \mbox{ for } i=1, \cdots, n \quad (\gamma_i) \\
		& S_{ij} = \Sigma_{ij}, \quad \mbox{ for } i,j=1, \cdots, n \quad (\lambda_{ij}) \\
		& p_i \ge 0, \quad \mbox{ for } i=1, \cdots, n \quad (\rho_i) \\
		&\begin{pmatrix} 
			u &   \boldsymbol{v}^T & \mathbf p^T\\
			\boldsymbol{v} & \mathbf{S}  & \mathbf Y^T \\
			\mathbf p & \mathbf Y & \mathbf X
		\end{pmatrix} \succeq 0 \quad (\mathbf{G}).
	\end{array}
\end{equation*}

The strict feasibility condition implies strong duality holds by Slater’s theorem. With the covariance matrix $\mathbf{\Sigma}_e$, the Lagrangian $L( \boldsymbol \alpha, \boldsymbol \beta, \nu, \boldsymbol \gamma, \boldsymbol \Lambda, \boldsymbol \rho, \mathbf{G}, u, \mathbf{v}, \mathbf{p},\mathbf Y, \mathbf X)$ is given by
\begin{eqnarray*}
 && \langle \mathbf{I}, \mathbf Y \rangle + \sum_{i=1}^m \alpha_i (b_i - \mathbf{a}_i^T \mathbf{p}) + \sum_{i=1}^m \beta_i (b_i^2 - \mathbf{a}_i^T \mathbf{X} \mathbf{a}_i) \\
	&& \qquad + \nu (1 - u) + \sum_{i=1}^n \gamma_i ( - v_i) +\langle \mathbf{\Lambda}, \mathbf{\Sigma}_e - \mathbf{S} \rangle + \sum_{i=1}^n \rho_i p_i \\
	&& \qquad +\left\langle 
	\begin{pmatrix} 
		g_1 &   \boldsymbol{g}_2^T & \mathbf g_3^T\\
		\boldsymbol{g}_2 & \mathbf{G}_{22}  & \mathbf G_{32}^T \\
		\mathbf g_3 & \mathbf G_{32} & \mathbf G_{33}
	\end{pmatrix} , 
	\begin{pmatrix} 
		u &   \boldsymbol{v}^T & \mathbf p^T\\
		\boldsymbol{v} & \mathbf{S}  & \mathbf Y^T \\
		\mathbf p & \mathbf Y & \mathbf X
	\end{pmatrix}
	\right\rangle.
\end{eqnarray*}

To minimize the Lagrangian, note that
it is always possible to find $\mathbf{G}$ that is positive semidefinite such that the inner product goes to negative infinity, unless the primal decision matrix is itself positive semidefinite. 
The Lagrangian dual can then be written as $$\min_{ \boldsymbol{\alpha}, \boldsymbol{\beta}, \nu, \boldsymbol{\gamma}, \mathbf{\Lambda}} \max_{u, \mathbf{v}, \mathbf{p},\mathbf Y, \mathbf X} L(\cdot),$$ subject to $\mathbf{G} \succeq 0$. Then, the inner maximization is bounded only if constraints \eqref{eq:dual-constraints} are satisfied. 
Thus, the dual can be rewritten as:
\begin{eqnarray*}
	\min_{\alpha, \beta, \nu, \gamma, \Lambda, \rho, \mathbf{G}}&& \boldsymbol{\alpha}^T \mathbf{b} + \boldsymbol{\beta}^T \mathbf{b}^2 + g_1 +  \langle 
	\mathbf{\Lambda}, \mathbf{\Sigma}_e \rangle, 
	\mbox{s.t.} 
	\mathbf{G} \succeq 0, \mbox{ and } \eqref{eq:dual-constraints}.
\end{eqnarray*}
Embedding this into the outer minimization problem yields \eqref{eq:min-min}. \qedhere

\subsection{Proof of Theorem \ref{thm:exact}}
We first rewrite $Z_P$ as below:
\begin{align*}
    Z_P 
    & =\sup_{\bmt z \sim (\bm \mu, \bm \Sigma)} \max \left\{\mathbb{E} [H(\bmt{z})],\mathbb{E} [-H(\bmt{z})]\right\}\\
    & = \max \left\{\sup_{\bmt z \sim (\bm \mu, \bm \Sigma)} \mathbb{E} [H(\bmt{z})],\sup_{\bmt z \sim (\bm \mu, \bm \Sigma)} \mathbb{E} [-H(\bmt{z})]\right\}.
\end{align*}

Define $Z_P^+=\sup_{\bmt z \sim (\bm \mu, \bm \Sigma)} \mathbb{E} [H(\bmt{z})]$ and $Z_P^-=\sup_{\bmt z \sim (\bm \mu, \bm \Sigma)} \mathbb{E} [-H(\bmt{z})]$. Then, $Z_P=\max\{Z_P^+,Z_P^-\}$. In the following, we prove $Z_P^+=Z_C^+$ following similar results in \cite{natarajan2011mixed} and \cite{kong2020appointment}; the proof for $Z_P^-=Z_C^-$ will be similar. 

\begin{lemma}\label{lm:relax}
	$Z_C^+ \ge Z_P^+$.
\end{lemma}
\begin{proof}
Any feasible solution $(w(\bmt z), y(\bmt z), s(\bmt z))$ to $Z_P^+$ satisfies
\begin{align*}
	\mathbf A \mathbf y(\bmt z) - \mathbf s(\bmt z) = \bmt z, 
	\mathbf A^T \mathbf w(\bmt z)=\mathbf b.
\end{align*}

The first constraint, i.e., $\mathbf a_i^T\mathbf y_i(\bmt z)=\tilde{z}_i+s_i(\bmt z) \forall i$, leads to $$\mathbb{E}[(\mathbf a_i^T\mathbf y_i(\bmt z))^2]=\mathbb{E}[(\tilde{z}_i+s_i(\bmt z))^2], \mbox{ and }$$  $$ \mathbb{E}[(\mathbf a_i^T\mathbf y_i(\bmt z))^2]=\mathbb{E}[(\mathbf a_i^T\mathbf y_i(\bmt z))(\tilde{z}_i+s_i(\bmt z))], \forall i.$$ Rearranging the terms, we have
\begin{align*}
	& \langle \mathbf A^T \mathbf A, \mathbb{E}[\mathbf y(\bmt z)\mathbf y(\bmt z)^T]\rangle=\langle \mathbf A^T, \mathbb{E}[\mathbf y(\bmt z)(\mathbf s(\bmt z)+\bmt z)^T]\rangle \\
	& \langle \mathbf A^T \mathbf A, \mathbb{E}[\mathbf y(\bmt z)\mathbf y(\bmt z)^T]\rangle=\langle \mathbf I, \mathbb{E}[(\mathbf s(\bmt z)+\bmt z)(\mathbf s(\bmt z)+\bmt z)^T]\rangle.
\end{align*}

From the second constraint, we have $$\mathbb{E}[\mathbf a_i^T\mathbf w(\bmt z)]=b_i, \mbox{ and } \mathbb{E}[(\mathbf a_i^T\mathbf w(\bmt z))^2]=\mathbf a_i^T \mathbb{E}[\mathbf w(\bmt z)\mathbf w(\bmt z)^T] \mathbf a_{i}=b_i^2, \forall i.$$

Let $\mathbf p_{w}=\mathbb{E}[\mathbf w(\bmt z)]$, $\mathbf p_{y}=\mathbb{E}[\mathbf y(\bmt z)]$, $\mathbf p_{s}=\mathbb{E}[\mathbf s(\bmt z)]$, $\mathbf X_{w}=\mathbb{E}[\mathbf w(\bmt z)\mathbf w(\bmt z)^T]$, $\mathbf X_{y}=\mathbb{E}[\mathbf y(\bmt z)\mathbf y(\bmt z)^T]$, $\mathbf X_{s}=\mathbb{E}[\mathbf s(\bmt z)\mathbf s(\bmt z)^T]$, $\mathbf Y_{w}=\mathbb{E}[\mathbf w(\bmt z)\bmt z^T]$, $\mathbf Y_{y}=\mathbb{E}[\mathbf y(\bmt z)\bmt z^T]$,  $\mathbf Y_{s}=\mathbb{E}[\mathbf s(\bmt z)\bmt z^T]$, $\mathbf Z_{yw}=\mathbb{E}[\mathbf w(\bmt z)\mathbf w(\bmt z)^T]$, $\mathbf Z_{sw}=\mathbb{E}[\mathbf s(\bmt z)\mathbf w(\bmt z)^T]$, $\mathbf Z_{sy}=\mathbb{E}[\mathbf s(\bmt z)\mathbf y(\bmt z)^T]$. They satisfy 
\begin{align*}
\left(\begin{array}{ccccc}
	1 & \bm \mu^{T} & \mathbf p_w^{T} & \mathbf p_y^{T} & \mathbf p_s^{T} \\
	\bm \mu & \bm \Sigma & \mathbf Y_w^{T} & \mathbf Y_y^T & \mathbf Y_s^T\\
	\mathbf p_w & \mathbf Y_w & \mathbf X_w & \mathbf Z_{yw}^T & \mathbf Z_{sw}^T\\
	\mathbf p_y & \mathbf Y_y & \mathbf Z_{yw} & \mathbf X_y & \mathbf Z_{sy}^T \\
	\mathbf p_s & \mathbf Y_s & \mathbf Z_{sw} & \mathbf Z_{sy} & \mathbf X_s \\
\end{array}\right)=\mathbb{E}\left[\left(\begin{array}{c}
1 \\
\bmt z \\
\mathbf w(\bmt z) \\
\mathbf y(\bmt z) \\
\mathbf s(\bmt z)
\end{array}\right)\left(\begin{array}{c}
1 \\
\bmt z \\
\mathbf w(\bmt z) \\
\mathbf y(\bmt z) \\
\mathbf s(\bmt z)
\end{array}\right)^{T}\right].
\end{align*}
Thus, $(\mathbf p_{w},\mathbf p_{y},\mathbf p_{s},\mathbf X_{w},\mathbf X_{y},\mathbf X_{s},\mathbf Y_{w},\mathbf Y_{y},\mathbf Y_{s},\mathbf Z_{yw},\mathbf Z_{sw},\mathbf Z_{sy})$ is a feasible solution to $Z_C^+$ with the objective value
$\langle \mathbf{VV}^T, \mathbf Y_w\rangle-\mathbf b^T \mathbf p_y=\langle \mathbf{VV}^T, \mathbb{E}[\mathbf w(\bmt z)\bmt z^T]\rangle-\mathbf b^T \mathbb{E}[\mathbf y(\bmt z)]=\mathbb{E} [H(\bmt{z})]$. That is, $Z_C^+$ is a relaxation of $Z_P^+$ and we have $Z_C^+ \ge Z_P^+$.
\qedhere
\end{proof}

We then consider a feasible solution to $Z_C^+$. By definition, it satisfies
\begin{align*}
	\left(\begin{array}{ccccc}
		1 & \bm \mu^{T} & \mathbf p_w^{T} & \mathbf p_y^{T} & \mathbf p_s^{T} \\
		\bm \mu & \bm \Sigma & \mathbf Y_w^{T} & \mathbf Y_y^T & \mathbf Y_s^T\\
		\mathbf p_w & \mathbf Y_w & \mathbf X_w & \mathbf Z_{yw}^T & \mathbf Z_{sw}^T\\
		\mathbf p_y & \mathbf Y_y & \mathbf Z_{yw} & \mathbf X_y & \mathbf Z_{sy}^T \\
		\mathbf p_s & \mathbf Y_s & \mathbf Z_{sw} & \mathbf Z_{sy} & \mathbf X_s \\
	\end{array}\right)=\sum_{l\in L}\left(\begin{array}{l}
\alpha_l \\
\bm \beta_l \\
\bm \gamma_l^w \\
\bm \gamma_l^y \\
\bm \gamma_l^s
\end{array}\right)\left(\begin{array}{l}
\alpha_l \\
\bm \beta_l \\
\bm \gamma_l^w \\
\bm \gamma_l^y \\
\bm \gamma_l^s
\end{array}\right)^{T}.
\end{align*}
Define $L_+=\{l\in L | \alpha_l >0\}$ and $L_0=\{l\in L | \alpha_l =0\}$.

\begin{lemma}\label{lm:feasible_y_s}
The above decomposition satisfies $A\bm\gamma_l^y-\bm \gamma_l^s=\bm \beta_l,\forall l\in L$.
\end{lemma}
\begin{proof}
By the decomposition, we have 
\begin{eqnarray*}
    \bm \Sigma=\sum_{l\in L} \bm{\beta}_l\bm{\beta}_l^T,
    \mathbf X_y=\sum_{l\in L}\bm{\gamma}_l^y{\bm{\gamma}_l^y}^T,
    \mathbf X_s=\sum_{l\in L}\bm{\gamma}_l^s{\bm{\gamma}_l^s}^T,\\
    \mathbf Y_y=\sum_{l\in L}\bm{\gamma}_l^y{\bm{\beta}_l}^T,
    \mathbf Y_s=\sum_{l\in L}\bm{\gamma}_l^s{\bm{\beta}_l}^T,
    \mathbf Z_{sy}^T=\sum_{l\in L}\bm\gamma_l^y {\bm \gamma_l^s}^T.
\end{eqnarray*}

The constraint $\langle \mathbf A^T \mathbf A, \mathbf X_{y}\rangle=\langle \mathbf A^T, \mathbf Y_{y}+\mathbf Z_{sy}^T\rangle$ in $Z_C^+$ is thus equivalent to $\langle \mathbf A^T \mathbf A, \sum_{l\in L}\bm{\gamma}_l^y{\bm{\gamma}_l^y}^T\rangle=\langle \mathbf A^T, \sum_{l\in L}\bm{\gamma}_l^y{\bm{\beta}_l}^T+\sum_{l\in L}\bm\gamma_l^y {\bm \gamma_l^s}^T\rangle$, which can be further rearranged into
\[\sum_{l\in L}(\mathbf A \bm{\gamma}_l^y)^T (\mathbf A \bm{\gamma}_l^y)= \sum_{l\in L}(\bm{\beta}_l+\bm \gamma_l^s)^T(\mathbf A \bm{\gamma}_l^y).\]

Moreover, the constraint $\langle \mathbf A^T \mathbf A, \mathbf X_{y}\rangle=\langle \mathbf I, \mathbf X_{s}+2 \mathbf Y_{s}+\bm \Sigma\rangle$ is equivalent to $\langle \mathbf A^T \mathbf A, \sum_{l\in L}\bm{\gamma}_l^y{\bm{\gamma}_l^y}^T\rangle=\langle \mathbf I, \sum_{l\in L}\bm{\gamma}_l^s{\bm{\gamma}_l^s}^T+2 \sum_{l\in L}\bm{\gamma}_l^s{\bm{\beta}_l}^T+\sum_{l\in L} \bm{\beta}_l\bm{\beta}_l^T\rangle$, which can be further rearranged into
\[\sum_{l\in L}(\mathbf A \bm{\gamma}_l^y)^T (\mathbf A \bm{\gamma}_l^y)= \sum_{l\in L}(\bm{\beta}_l+\bm \gamma_l^s)^T(\bm{\beta}_l+\bm \gamma_l^s).\]

Combining both equations above, we have
\begin{align*}
&\sum_{l\in L}(\mathbf A \bm{\gamma}_l^y-(\bm{\beta}_l+\bm \gamma_l^s))^T (\mathbf A \bm{\gamma}_l^y-(\bm{\beta}_l+\bm \gamma_l^s))\\
&=\sum_{l\in L}(\mathbf A \bm{\gamma}_l^y)^T (\mathbf A \bm{\gamma}_l^y)-2\sum_{l\in L}(\bm{\beta}_l+\bm \gamma_l^s)^T(\mathbf A \bm{\gamma}_l^y)+\sum_{l\in L}(\bm{\beta}_l+\bm \gamma_l^s)^T(\bm{\beta}_l+\bm \gamma_l^s)\\
&=0.
\end{align*}
Note that $(\mathbf A \bm{\gamma}_l^y-(\bm{\beta}_l+\bm \gamma_l^s))^T (\mathbf A \bm{\gamma}_l^y-(\bm{\beta}_l+\bm \gamma_l^s))\ge 0$. We conclude $(\mathbf A \bm{\gamma}_l^y-(\bm{\beta}_l+\bm \gamma_l^s))^T (\mathbf A \bm{\gamma}_l^y-(\bm{\beta}_l+\bm \gamma_l^s))= 0$, i.e., $\mathbf A \bm{\gamma}_l^y-\bm \gamma_l^s=\bm{\beta}_l, \forall l\in L$.
\qedhere
\end{proof}

\begin{lemma}\label{lm:feasible_w}
	The decomposition satisfies $\mathbf{A}^T \frac{\bm \gamma_l^w}{\alpha_l} = \mathbf b, \forall l\in L_+$ and $\bm \gamma_l^w=\mathbf 0, \forall l \in L_0$.
\end{lemma}
\begin{proof}
By the decomposition, we have $\mathbf p_w=\sum_{l\in L}\alpha_l\gamma_l^w$ and $\mathbf X_w=\sum_{l\in L}\bm{\gamma}_l^w{\bm{\gamma}_l^w}^T$. The first two constraints $\mathbf a_i^T \mathbf p_{w}=b_i$ and $\mathbf a_i^T \mathbf X_{w} \mathbf a_{i}=b_{i}^{2}$ in $Z_C^+$ are equivalent to
\[\mathbf a_i^T \sum_{l\in L}\alpha_l\gamma_l^w=b_i \mbox{ and } \mathbf a_i^T \sum_{l\in L}\bm{\gamma}_l^w{\bm{\gamma}_l^w}^T \mathbf a_{i}=b_{i}^{2}
\]

Note that $\sum_{l\in L}\alpha_l^2=1$. The above equations lead to
\[\left(\sum_{l\in L}\alpha_l^2\right)\left(\sum_{l\in L}(\mathbf a_i^T \bm{\gamma}_l^w)^2\right)=b_i^2=\left(\sum_{l\in L}\alpha_l\mathbf a_i^T \gamma_l^w\right)^2.\]
By the equality conditions of Cauchy-Schwartz inequality, for each $i$, $\exists\zeta_i$ such that $\zeta_i \alpha_l=\mathbf a_i^T \bm{\gamma}_l^w, \forall l \in L$. Recall that $b_i= \sum_{l\in L}\alpha_l\mathbf a_i^T\gamma_l^w=\zeta_i \sum_{l\in L}\alpha_l^2=\zeta_i$. Then, for all $l\in L_+$, we have $\mathbf a_i^T \frac{\bm{\gamma}_l^w}{\alpha_l}=b_i$, which leads to $\mathbf{A}^T \frac{\bm \gamma_l^w}{\alpha_l} = \mathbf b, \forall l\in L_+$; for all $l\in L_0$, we have $\mathbf a_i^T \bm{\gamma}_l^w=b_i\alpha_k=0$, which leads to $\mathbf A^T \bm{\gamma}_l^w=\mathbf 0, \forall l\in L_0$. Suppose the feasible region $\boldsymbol{P} = \{\mathbf{w} \ge 0 | \mathbf{A}^T \mathbf{w} = b \}$ for the second stage problem is bounded and nonempty, then we conclude $\bm{\gamma}_l^w=\mathbf 0, \forall l\in L_0$.\qedhere
\end{proof}
By the above lemma, we can decompose the optimal solution to $Z_C^+$ as
\begin{align*}
	&\left(\begin{array}{ccccc}
		1 & \bm \mu^{T} & \mathbf p_w^{*T} & \mathbf p_y^{*T} & \mathbf p_s^{*T} \\
		\bm \mu & \bm \Sigma & \mathbf Y_w^{*T} & \mathbf Y_y^{*T} & \mathbf Y_s^{*T}\\
		\mathbf p_w^* & \mathbf Y_w^* & \mathbf X_w^* & \mathbf Z_{yw}^{*T} & \mathbf Z_{sw}^{*T}\\
		\mathbf p_y^* & \mathbf Y_y^* & \mathbf Z_{yw}^* & \mathbf X_y^* & \mathbf Z_{sy}^{*T} \\
		\mathbf p_s^* & \mathbf Y_s^* & \mathbf Z_{sw}^* & \mathbf Z_{sy}^* & \mathbf X_s^* \\
	\end{array}\right)\\
 &=\sum_{l\in L_+}\alpha_l^2\left(\begin{array}{l}
		1 \\
		\frac{\bm \beta_l}{\alpha_l} \\
		\frac{\bm \gamma_l^w}{\alpha_l} \\
		\frac{\bm \gamma_l^y}{\alpha_l} \\
		\frac{\bm \gamma_l^s}{\alpha_l}
	\end{array}\right)\left(\begin{array}{l}
		1 \\
		\frac{\bm \beta_l}{\alpha_l} \\
		\frac{\bm \gamma_l^w}{\alpha_l} \\
		\frac{\bm \gamma_l^y}{\alpha_l} \\
		\frac{\bm \gamma_l^s}{\alpha_l}
	\end{array}\right)^{T}+\sum_{l\in L_0}\left(\begin{array}{l}
	0 \\
	\bm \beta_l \\
	\bm 0 \\
	\bm \gamma_l^y \\
	\bm \gamma_l^s
	\end{array}\right)\left(\begin{array}{l}
	0 \\
	\bm \beta_l \\
	\bm 0 \\
	\bm \gamma_l^y \\
	\bm \gamma_l^s
	\end{array}\right)^{T}.
\end{align*}

Let $\epsilon\in(0,1)$. We define a sequence of random variables $\tilde{\bm{z}}_\epsilon$ and their corresponding feasible solutions $\mathbf{w}^*\left(\tilde{\bm z}_{\epsilon}\right), \mathbf{y}^*\left(\tilde{\bm z}_{\epsilon}\right),\mathbf{s}^*\left(\tilde{\bm z}_{\epsilon}\right)$ (by Lemmas~\ref{lm:feasible_y_s} and~\ref{lm:feasible_w}) with the following distribution
\begin{align*}
	&\mathbb{P}\left(\left(\tilde{\bm z}_{\epsilon}, \mathbf{w}^*\left(\tilde{\bm z}_{\epsilon}\right), \mathbf{y}^*\left(\tilde{\bm z}_{\epsilon}\right),\mathbf{s}^*\left(\tilde{\bm z}_{\epsilon}\right)\right)=\left(\frac{\boldsymbol{\beta}_{l}}{\alpha_{l}}, \frac{\boldsymbol{\gamma}_{l}^w}{\alpha_{l}}, \frac{\boldsymbol{\gamma}_{l}^y}{\alpha_{l}}, \frac{\boldsymbol{\gamma}_{l}^s}{\alpha_{l}}\right)\right)\\
 &=\left(1-\epsilon^{2}\right) \alpha_{l}^{2}, \forall l \in L_{+}, \\
	&\mathbb{P}\left(\left(\tilde{\bm z}_{\epsilon}, \mathbf{w}^*\left(\tilde{\bm z}_{\epsilon}\right), \mathbf{y}^*\left(\tilde{\bm z}_{\epsilon}\right),\mathbf{s}^*\left(\tilde{\bm z}_{\epsilon}\right)\right)=\left(\frac{\sqrt{|L_0|}\boldsymbol{\beta}_{l}}{\epsilon}, \mathbf{w}^\dagger, \frac{\sqrt{|L_0|}\boldsymbol{\gamma}_{l}^y}{\epsilon}, \frac{\sqrt{|L_0|}\boldsymbol{\gamma}_{l}^s}{\epsilon}\right)\right)\\
 &=\frac{\epsilon^2}{\left|L_0\right|}, \forall l \in L_0,
\end{align*}
where $\mathbf{w}^\dagger\in \mathbf{P}$ is any feasible solution $\mathbf w$ to $H(\bmt z)$.

We can verify that 
\begin{align*}
&\mathbb{E}\left[\left(\begin{array}{c}
	1\\
	\tilde{\bm z}_{\epsilon}\\
	\mathbf{w}^*\left(\tilde{\bm z}_{\epsilon}\right)\\
	\mathbf{y}^*\left(\tilde{\bm z}_{\epsilon}\right)\\
	\mathbf{s}^*\left(\tilde{\bm z}_{\epsilon}\right)
\end{array}\right)\left(\begin{array}{c}
	1\\
	\tilde{\bm z}_{\epsilon}\\
	\mathbf{w}^*\left(\tilde{\bm z}_{\epsilon}\right)\\
	\mathbf{y}^*\left(\tilde{\bm z}_{\epsilon}\right)\\
	\mathbf{s}^*\left(\tilde{\bm z}_{\epsilon}\right)
\end{array}\right)^{T}\right]\\
&=\sum_{l\in L_+}(1-\epsilon^2)\alpha_l^2\left(\begin{array}{c}
1 \\
\frac{\bm \beta_l}{\alpha_l} \\
\frac{\bm \gamma_l^w}{\alpha_l} \\
\frac{\bm \gamma_l^y}{\alpha_l} \\
\frac{\bm \gamma_l^s}{\alpha_l}
\end{array}\right)\left(\begin{array}{c}
1 \\
\frac{\bm \beta_l}{\alpha_l} \\
\frac{\bm \gamma_l^w}{\alpha_l} \\
\frac{\bm \gamma_l^y}{\alpha_l} \\
\frac{\bm \gamma_l^s}{\alpha_l}
\end{array}\right)^{T}+\sum_{l\in L_0}\frac{\epsilon^2}{\left|L_0\right|}\left(\begin{array}{c}
1 \\
\frac{\sqrt{|L_0|}\bm \beta_l}{\epsilon} \\
\mathbf{w}^\dagger \\
\frac{\sqrt{|L_0|}\bm \gamma_l^y}{\epsilon} \\
\frac{\sqrt{|L_0|}\bm \gamma_l^s}{\epsilon}
\end{array}\right)\left(\begin{array}{c}
1 \\
\frac{\sqrt{|L_0|}\bm \beta_l}{\epsilon} \\
\mathbf{w}^\dagger \\
\frac{\sqrt{|L_0|}\bm \gamma_l^y}{\epsilon} \\
\frac{\sqrt{|L_0|}\bm \gamma_l^s}{\epsilon}
\end{array}\right)^{T}\\
&\overrightarrow{\epsilon \downarrow 0} \left(\begin{array}{ccccc}
	1 & \bm \mu^{T} & \mathbf p_w^{*T} & \mathbf p_y^{*T} & \mathbf p_s^{*T} \\
	\bm \mu & \bm \Sigma & \mathbf Y_w^{*T} & \mathbf Y_y^{*T} & \mathbf Y_s^{*T}\\
	\mathbf p_w^* & \mathbf Y_w^* & \mathbf X_w^* & \mathbf Z_{yw}^{*T} & \mathbf Z_{sw}^{*T}\\
	\mathbf p_y^* & \mathbf Y_y^* & \mathbf Z_{yw}^* & \mathbf X_y^* & \mathbf Z_{sy}^{*T} \\
	\mathbf p_s^* & \mathbf Y_s^* & \mathbf Z_{sw}^* & \mathbf Z_{sy}^* & \mathbf X_s^* \\
\end{array}\right)
\end{align*}

As $\epsilon \downarrow 0$, the random vectors $\left(\tilde{\bm z}_{\epsilon}, \mathbf{w}^*\left(\tilde{\bm z}_{\epsilon}\right), \mathbf{y}^*\left(\tilde{\bm z}_{\epsilon}\right),\mathbf{s}^*\left(\tilde{\bm z}_{\epsilon}\right)\right)$ converge almost surely (a.s.) to $\left(\tilde{\bm z}^*, \mathbf{w}^*\left(\tilde{\bm z}^*\right), \mathbf{y}^*\left(\tilde{\bm z}^*\right),\mathbf{s}^*\left(\tilde{\bm z}^*\right)\right)$ with distribution below: \[\mathbb{P}\left(\left(\tilde{\bm z}^*, \mathbf{w}^*\left(\tilde{\bm z}^*\right), \mathbf{y}^*\left(\tilde{\bm z}^*\right),\mathbf{s}^*\left(\tilde{\bm z}^*\right)\right)=\left(\frac{\boldsymbol{\beta}_{l}}{\alpha_{l}}, \frac{\boldsymbol{\gamma}_{l}^w}{\alpha_{l}}, \frac{\boldsymbol{\gamma}_{l}^y}{\alpha_{l}}, \frac{\boldsymbol{\gamma}_{l}^s}{\alpha_{l}}\right)\right)= \alpha_{l}^{2}, \forall l\in L_+.\]

We then have
\begin{align*}
\lim _{\epsilon \downarrow 0} \mathbb{E}\left[H(\bmt z_\epsilon)\right]
&=\lim _{\epsilon \downarrow 0} \left(\sum_{l\in L_+}(1-\epsilon^2)\alpha_l^2\left(\mathbf{V} \mathbf{V}^T  \frac{\bm{\gamma}_{l}^w}{\alpha_{l}}\frac{\bm{\beta}_{l}}{\alpha_{l}} -\mathbf{b}^T \frac{\bm{\gamma}_l^y}{\alpha_l}\right) \right. \\
&\quad \left. +\sum_{l\in L_0}\frac{\epsilon^2}{\left|L_0\right|}\left(\mathbf{V} \mathbf{V}^T  \mathbf{w}^\dagger\frac{\sqrt{|L_0|}\boldsymbol{\beta}_{l}}{\epsilon} -\mathbf{b}^T \frac{\sqrt{|L_0|}\boldsymbol{\gamma}_{l}^y}{\epsilon}\right)\right)\\
&=\sum_{l\in L_+}\left(\mathbf{V} \mathbf{V}^T  \bm{\gamma}_{l}^w\bm{\beta}_{l} -\mathbf{b}^T \bm{\gamma}_l^y\alpha_l\right)=\mathbb{E}\left[H(\bmt z^*)\right].
\end{align*}

By the similar limiting argument on page 719 of \cite{natarajan2011mixed}, we have $Z_P^+=\sup_{\bmt z \sim (\bm \mu, \bm \Sigma)} \mathbb{E} [H(\bmt{z})]\ge \mathbb{E}\left[H(\bmt z^*)\right]$ and thus
\begin{align*}
	Z_P^+&=\sup_{\bmt z \sim (\bm \mu, \bm \Sigma)} \mathbb{E} [H(\bmt{z})]\\
 &\ge \mathbb{E}\left[H(\bmt z^*)\right]=\mathbf{V} \mathbf{V}^T  \sum_{l\in L_+}\bm{\gamma}_{l}^w\bm{\beta}_{l} -\mathbf{b}^T\sum_{l\in L_+} \bm{\gamma}_l^y\alpha_l\\
 &\quad =\langle \mathbf{VV}^T, \mathbf Y_w^*\rangle-\mathbf b^T \mathbf p_y^*=Z_C^+.
\end{align*}

Combining with Lemma~\ref{lm:relax}, we conclude $Z_C^+=Z_P^+$. To prove $Z_C^-=Z_P^-$, we only need to change the sign of the objective and the proof follows as $Z_C^-$ and $Z_C^+$ share the same constraints.

\end{document}